\documentclass[10pt,journal,compsoc]{IEEEtran}

\usepackage{cite}
\usepackage[cmex10]{amsmath}
\usepackage{amssymb,amsfonts}
\usepackage{bm}
\usepackage{graphicx}
\usepackage{booktabs}
\usepackage{array}
\usepackage{multirow}
\usepackage{url}
\usepackage{balance}
\usepackage{algorithm}
\usepackage{algorithmic}

\newtheorem{proposition}{Proposition}
\newtheorem{theorem}{Theorem}
\newtheorem{corollary}{Corollary}

\newcommand{\lhat}{\hat\lambda}
\newcommand{\deff}{d_{\mathrm{eff}}}
\newcommand{\muhat}{\hat{\bm\mu}}
\newcommand{\mustar}{\bm\mu^{*}}
\newcommand{\vt}{\bm t}
\newcommand{\CI}[2]{[#1,\,#2]}
\newcommand{\CIonly}[2]{[#1,\,#2]}
\newcommand{\PCIN}{$+1.52$ \CI{+0.36}{+2.68}}
\newcommand{\PCOUT}{$-0.52$ \CI{-2.24}{+1.20}}
\newcommand{\CLAPPC}{$+2.18$ \CI{+0.17}{+4.20}}
\newcommand{\OFFQ}{$-0.63$ to $-0.26$\,pp}
\newcommand{\OFFH}{$-2.21$ to $-0.84$\,pp}
\newcommand{\DNORM}{1.02}
\newcommand{\DCREAL}{$-9.67$ \CI{-14.37}{-4.96}}
\newcommand{\DCNOD}{$-7.17$ \CI{-10.86}{-3.47}}
\newcommand{\DCREC}{$+2.50$ \CI{+0.91}{+4.09}}
\newcommand{\DCSHARE}{26\%}
\newcommand{\DCCELLS}{114}

\newcommand{\AUGCLAP}{$-0.08$ \CI{-0.52}{+0.36}}
\newcommand{\AUGLPP}{$-0.98$ \CI{-2.13}{+0.18}}
\newcommand{\AUGHOSO}{$+0.28$ \CI{+0.04}{+0.52}}

\newcommand{\AUGGDAKone}{$-18.83$ \CI{-30.66}{-6.99}}
\newcommand{\AUGGDAKtwo}{$+4.28$ \CI{+1.19}{+7.36}}
\newcommand{\COOPCLAP}{$+2.26$ \CI{+1.04}{+3.49}}

\newcommand{\COOPLOO}{$-1.64$ \CI{-2.93}{-0.35}}
\newcommand{\COOPJS}{$-7.68$ \CI{-13.70}{-1.67}}
\newcommand{\COOPSHARE}{$80.3\%$ \CI{74.8}{85.8}}
\newcommand{\AUGCELLS}{817}
\newcommand{\COOPCELLS}{349}
\newcommand{\AUGDS}{10}

\newcommand{\AUGPAIRS}{39 of the 50}
\newcommand{\COOPPAIRS}{23 of the 50}

\newcommand{\TIERPHOTOCLAP}{$+2.11$ \CI{+1.21}{+3.01}}
\newcommand{\TIERPHOTOLPP}{$+1.86$ \CI{+0.82}{+2.89}}
\newcommand{\TIERENSCLAP}{$+1.85$ \CI{+0.85}{+2.86}}
\newcommand{\TIERENSLPP}{$+1.28$ \CI{+0.21}{+2.34}}
\newcommand{\TIERCUPLCLAP}{$+1.67$ \CI{+0.73}{+2.60}}
\newcommand{\TIERCUPLLPP}{$+1.28$ \CI{+0.22}{+2.33}}

\newcommand{\TIERCELLS}{3{,}600}
\newcommand{\TIERDS}{10}

\newcommand{\CAPRATIO}{55\%}
\newcommand{\CAPSPEAR}{+0.61}
\newcommand{\CAPSPEARCI}{\CI{+0.19}{+1.00}}
\newcommand{\CAPPARTIAL}{+0.52}
\newcommand{\CAPWDS}{+0.70}
\newcommand{\CAPWBK}{+0.62}
\newcommand{\DEFFMED}{32}
\newcommand{\DEFFRANGE}{11--45}
\newcommand{\SECLOSEKtwo}{1.15}
\newcommand{\SECLOSEKsixteen}{1.09}
\newcommand{\SEEDSDKtwo}{$2.4\times10^{-2}$}
\newcommand{\SEEDSDKsixteen}{$8.2\times10^{-4}$}
\newcommand{\AONEOPT}{9}

\newcommand{\SIGTWOCLAP}{$+1.29$ \CI{+0.17}{+2.42}}
\newcommand{\SIGTWOLPP}{$+1.50$ \CI{-0.22}{+3.22}}
\newcommand{\SIGTWOLPPALL}{$-0.16$ \CI{-1.45}{+1.14}}
\newcommand{\DINOTHREE}{$+7.21$ \CI{+0.41}{+14.01}}
\newcommand{\DINOCLIP}{$-2.24$ \CI{-6.25}{+1.76}}

\newcommand{\SIGTWOLAMJS}{0.97}
\newcommand{\SIGTWOLAMOR}{0.30}

\newcommand{\suppself}{This appendix}
\newcommand{\suppselflower}{this appendix}
\newcommand{\restate}[2]{\medskip\noindent\textbf{#1}\ \emph{#2}\medskip}

\begin{document}

\title{The Blending Ratio Is Not Where the Performance Is:\\
Diagnosing Prototype Blending for Few-Shot Adaptation of Vision--Language Models}

\author{Liangzhi~Li, Bowen~Wang, Yiming~Qian, Thorsten~Neumann, Xia~Xie, and~Guangshun~Li%
\IEEEcompsocitemizethanks{\IEEEcompsocthanksitem L.~Li is with the Department of Computer Science, Qufu Normal University, China, and with Climind (e-mail: conscienceli@gmail.com). B.~Wang is with the Institute of Scientific and Industrial Research (SANKEN), The University of Osaka, Japan (e-mail: wang@im.sanken.osaka-u.ac.jp). Y.~Qian is with the Institute of High Performance Computing (IHPC), A*STAR, Singapore. T.~Neumann is with Standard Chartered Bank, Singapore. X.~Xie is with the School of Aerospace Intelligence, Hainan University, China. G.~Li is with the School of Computer Science, Qufu Normal University, China (e-mail: guangshunli@qfnu.edu.cn).
\IEEEcompsocthanksitem \emph{(Corresponding author: Guangshun Li.)}}%
\thanks{\textbf{Artefacts:} the release carries the pipeline, per-cell feature caches under both split protocols, the augmented-view caches, and all $4{,}800$ records; \texttt{scripts/build.sh} rebuilds every table and figure on a CPU.}}

\markboth{}{Li \MakeLowercase{\textit{et al.}}: The Blending Ratio Is Not Where the Performance Is}

\IEEEtitleabstractindextext{%
\begin{abstract}
Many few-shot adaptation methods for vision--language models classify with a convex combination of the zero-shot text prototype and the mean of the $K$ labelled image features, governed by a single blending ratio routinely tuned on held-out labels, often on the test set itself. This paper asks the question that the family's own bias--variance justification invites: what is the \emph{right} value of the blending ratio, can it be estimated without validation data, and is finding it actually where the performance is? We first show that the ratio minimising prototype mean-squared error has a closed form whose support-set plug-in is exactly a positive-part James--Stein coefficient shrinking towards the text prototype. Across a 4{,}800-cell matrix (ten datasets, five backbones including SigLIP, five shot counts, five seeds, four prompt tiers; the 5{,}000-cell grid minus 200 thin-pool drops) this theoretically optimal ratio is a reliable estimate of the wrong quantity: on the $950$ cells of the primary prompt tier where it is defined it trails a test-set-oracle ratio by $8.5$ accuracy points. It saturates at $\lambda\!\approx\!1$, discarding the text prior and reproducing a nearest-class-mean classifier, because $78\%$ of the text--image prototype distance it treats as bias is a class-independent offset that the arg\,max largely cancels. We prove the mechanism and measure its share of the damage by a counterfactual ($26\%$). Second, a blend whose ratio is set by leave-one-out on the support set alone lands within $0.9$ points of the oracle-ratio blend's \emph{accuracy}, so the ratio \emph{is} estimable without validation data. Third, and most consequentially, validation-free linear probes beat even the oracle-tuned blend: CLAP by $+1.9$ points and LP++ by $+1.5$ on average, and at $K{\ge}4$ all four reproduced validation-free baselines sit above the oracle, the two linear probes by margins whose intervals exclude zero and that are invariant to solver precision. These results locate the family's ceiling in its model class, not its hyperparameter: the blending ratio can be set near-optimally for free, and it is still not where the performance is. Code, cached features, and all per-cell records: \url{https://huggingface.co/datasets/Liangzhi-Li/clipbench-blending}.
\end{abstract}

\begin{IEEEkeywords}
Vision--language models, CLIP, few-shot learning, transfer learning, shrinkage estimation, James--Stein, model selection, evaluation protocols.
\end{IEEEkeywords}}

\maketitle
\IEEEdisplaynontitleabstractindextext
\IEEEpeerreviewmaketitle

\IEEEraisesectionheading{\section{Introduction}\label{sec:intro}}

\begin{figure*}[!t]
\centering
\includegraphics[width=0.86\textwidth]{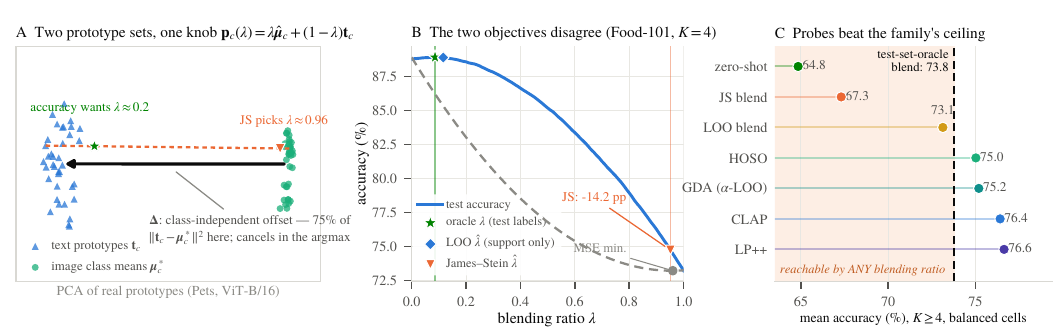}
\caption{\textbf{A:} PCA of the prototype geometry (Pets, ViT-B/16): the common offset $\bm\Delta$ (the modality gap) is $75\%$ of the squared text--image prototype distance here, yet costs far less at decision time (Thm.~\ref{thm:shift}; Sec.~\ref{sec:perclass}). Blending slides each prototype along its dotted path; marks are the two objectives' choices. \textbf{B:} measured accuracy and prototype MSE vs.\ $\lambda$ (Food-101, $K{=}4$): the MSE minimum ($\lambda\approx0.96$, where Eq.~\eqref{eq:jshat} sits) costs $14.2$ accuracy points against the accuracy peak ($\lambda\approx0.09$, which leave-one-out finds). \textbf{C:} mean accuracy, $K\ge4$, balanced cells: every reproduced validation-free baseline exceeds the \emph{test-set-oracle} blend (dashed): the family is capacity-limited, not mis-tuned.}
\label{fig:teaser}
\end{figure*}

\IEEEPARstart{C}{ontrastively} pre-trained vision--language models such as CLIP~\cite{radford2021clip}, ALIGN~\cite{jia2021align}, and SigLIP~\cite{zhai2023siglip} turned image classification into retrieval. Given a few labelled images per class, the natural upgrade averages each class's $K$ image features into an image prototype and classifies by a convex combination of the two modalities,
$\bm p_c(\lambda) = \lambda\,\muhat_c + (1-\lambda)\,\vt_c$,
with $\vt_c$ the text prototype, $\muhat_c$ the empirical image prototype, and $\lambda\in[0,1]$ a \emph{blending ratio} deciding how far to trust the few labelled shots over the frozen prior (Fig.~\ref{fig:teaser}A). This underlies much of the few-shot adaptation literature: cache models blend zero-shot logits with support-set affinities~\cite{zhang2022tip,zhu2023ape,udandarao2023susx,zhang2023cafo}, adapters residual with original features~\cite{gao2024clipadapter}, task-residual methods add a learned offset to the text prototype~\cite{yu2023taskres}, Gaussian discriminant heads blend generative with zero-shot logits~\cite{wang2024gda}, training-free prototype methods the two modality means~\cite{goswami2026protomix,lin2023crossmodal}. Each has a knob---$\alpha$, $\lambda$, a residual ratio, a mixing coefficient---arbitrating between prior and evidence.

\looseness=-1 How is that knob set? Almost never from the $K$ shots alone. The original Tip-Adapter searches its $\alpha$ and $\beta$ on the test set~\cite{zhang2022tip}; the CoOp-inherited few-shot CLIP protocol~\cite{zhou2022coop} ships extra validation images that strict few-shot learning lacks; the official ``hard-to-beat'' GDA implementation~\cite{wang2024gda} grid-searches its blending weight on held-out labels. A critical literature documents this: CLAP~\cite{silva2024closer} showed state-of-the-art adapters need per-task tuning on large validation sets to beat a linear probe; LP++~\cite{huang2024lppp} observed hyperparameters tuned ``on the entire test set''; HOSO~\cite{vorster2026hoso} formalised a validation-free protocol because most published numbers are not; benchmark critiques found the same pathology in test-time adaptation~\cite{sheng2025illusion} and few-shot transfer~\cite{luo2026fewtrans}.

\looseness=-1 The standard argument for blending, explicit in recent work~\cite{goswami2026protomix}, is statistical: the image prototype is an unbiased, high-variance estimate of the true class mean, the text prototype a biased, zero-variance one; blending trades bias against variance. If that frame is right, three questions should have quantitative answers. \emph{First}, what is the optimal blending ratio, and can it be computed rather than searched? \emph{Second}, can it be estimated from the $K$ support shots alone, with no validation data? \emph{Third}, the question nobody asks: if the ratio were set \emph{perfectly}, would the classifier be competitive? We answer all three at scale: 4{,}800 cells (a 5{,}000-cell grid minus 200 guarded drops, Sec.~\ref{sec:stats}) over ten datasets, five backbones, five shot counts, five support-set seeds, and four tiers of prompt quality, with four published validation-free baselines reproduced from official code and every comparison a paired difference with cluster-bootstrap confidence intervals.

\emph{(i) The theoretically optimal ratio is a reliable estimate of the wrong quantity.} Minimising prototype mean-squared error yields a closed-form ratio whose support-set plug-in is exactly a positive-part James--Stein coefficient~\cite{james1961estimation,efron1973stein} shrinking the image prototype towards the text prototype---and provably dominating it at that objective (Prop.~\ref{prop:dom}). Yet the classifier it induces trails a test-set-oracle ratio by $8.5$ accuracy points on average and is statistically indistinguishable from zero-shot. The reason is structural: we measure $78.3\%$ \CI{74.6}{82.1} of the squared text--image prototype distance as a \emph{class-independent} offset, the modality gap~\cite{liang2022mind}---one the arg\,max cancels but the MSE objective books as bias of the text prototype. On SigLIP the share rises to $86.6\%$: the phenomenon belongs to contrastive pre-training at large, not to CLIP. \emph{(ii) The ratio is estimable without validation data.} A leave-one-out estimator re-using all $K$ shots matches the test-set-oracle blend to within $0.9$ accuracy points (ratio rank correlation $0.79$) and responds correctly to interventions: upgrading the prompt from a generic template to LLM descriptions~\cite{pratt2023cupl} lowers the oracle ratio, and the estimate tracks the shift. \emph{(iii) Even the oracle-tuned blend is not competitive.} Validation-free linear probes beat the family's own unreachable upper bound: CLAP by $+1.92$ \CI{+0.87}{+2.98} points and LP++ by $+1.46$ \CI{+0.34}{+2.58} on average. At $K\ge4$ all four reproduced baselines are above the test-set-oracle blend in point estimate, three of them with intervals excluding zero. One caveat on scope: the Eq.~\eqref{eq:blend} members we evaluate are our own constructions (Sec.~\ref{sec:scope}), so the bound concerns the family's supremum, not any published system.

We believe this reframes how the prototype-blending literature should be read (Sec.~\ref{sec:discussion}): our results say the blending axis is worth at most the distance to the oracle, under one point for a well-estimated ratio, while the model-class gap to a plain linear probe on the same shots runs two to three points the other way. For practitioners this yields a validation-free recipe; for theory, it explains why the bias--variance story fails and supplies the decomposition that any future justification of blending will have to address.

\smallskip\noindent\textbf{Contributions.}
(1)~A closed form for the MSE-optimal blending ratio, identification of its support-set plug-in with positive-part James--Stein shrinkage, and its exact sampling theory (Prop.~\ref{prop:sampling}; Sec.~\ref{sec:theory}).
(2)~A proof (Theorem~\ref{thm:shift}, Corollaries~\ref{cor:tworisk}--\ref{cor:sstar}) and a 4{,}800-cell measurement that this ratio over-trusts image evidence by construction, because $78\%$ of the modality gap it penalises is class-independent and the arg\,max largely cancels (Secs.~\ref{sec:theory},~\ref{sec:results}; a counterfactual puts its share of the damage at $26\%$).
(3)~A leave-one-out ratio estimator using only the $K$ support shots, within $0.9$ accuracy points of a test-set oracle, plus a rule making the published GDA baseline validation-free (Sec.~\ref{sec:estimation}).
(4)~The capacity result, now with its theory: validation-free linear probes exceed the test-set-oracle blend (Sec.~\ref{sec:results}), the blend path provably cannot reach the unrestricted separation a probe can aim at (Theorem~\ref{thm:capacity}), and the theorem's predicted headroom picks out the strata where the measured surplus is largest.
(5)~A released benchmark---code, cached features, per-cell records, and line-by-line reproductions of CLAP, LP++, GDA, and HOSO---with paired statistics and a documented audit trail of every protocol decision (Sec.~\ref{sec:protocol}).

\section{Related Work}\label{sec:related}

\subsection{Few-Shot Adaptation of Vision--Language Models}
Adapting a frozen vision--language model with $K$ labelled images per class has produced several method families. \emph{Prompt tuning} learns continuous text-side context vectors (CoOp~\cite{zhou2022coop}, CoCoOp~\cite{zhou2022cocoop}, MaPLe~\cite{khattak2023maple}), regularised to stay near the hand-crafted context~\cite{zhu2023prograd,yao2023kgcoop,chen2023plot}---the prior--evidence tension again, as a penalty rather than a ratio. \emph{Adapter methods} attach lightweight modules on cached features: CLIP-Adapter blends a bottleneck MLP's output with the original feature through a residual ratio~\cite{gao2024clipadapter}; Task Residual adds a learned offset to the text prototypes~\cite{yu2023taskres}. \emph{Cache-based methods} start from Tip-Adapter~\cite{zhang2022tip}, which adds an $\alpha$-weighted support-set affinity term to the zero-shot logits; descendants refine the cache~\cite{zhu2023ape,udandarao2023susx,zhang2023cafo,zhang2024dmn,karmanov2024tda}, and kernel analyses unify the branch as regularised regression~\cite{bendou2025proker,ding2024representer}. \emph{Generative heads} blend fitted Gaussian-discriminant-analysis logits with the zero-shot logits~\cite{wang2024gda}. \emph{Prototype methods} blend the modality means directly~\cite{lin2023crossmodal,goswami2026protomix}. Our subject cuts across these families: every one owns a scalar trading prior against evidence, and we analyse that scalar, not the architecture around it.

\subsection{Evaluation Protocols and the Validation-Set Problem}
The standard few-shot CLIP protocol inherited from CoOp quietly assumes labelled data beyond the $K$ shots. CLAP~\cite{silva2024closer} showed that leading adapters outperform a linear probe only when tuned per task on a large validation set, and proposed a class-adaptive probe needing none. LP++~\cite{huang2024lppp} derived data-driven step sizes from Lipschitz constants, making a strong linear probe without learning-rate search. HOSO~\cite{vorster2026hoso} named the setting, \emph{validation-free} few-shot adaptation, and fit CLIP-Adapter's blending ratio on one held-out shot; the critique recurs in the medical domain~\cite{silva2025fewshot}. At benchmark scale, Luo \emph{et al.}~\cite{luo2026fewtrans} argue that cross-validation on $K$ shots is unreliable and reported rankings are dominated by hyperparameter access; Sheng \emph{et al.}~\cite{sheng2025illusion} re-evaluate test-time adaptation (a parallel line adapting without support labels~\cite{shu2022tpt,karmanov2024tda,zhang2024dmn,zhu2024awt}) under equalised tuning and find most gains evaporate; Kravets \emph{et al.}~\cite{kravets2025rethinking} show that CLIP's pre-training corpus leaks class information into the benchmarks. Our work differs: rather than asking whether reported gains survive fair tuning, we bound what \emph{perfect} tuning could ever deliver, and show the bound is below a validation-free linear probe.

\subsection{Shrinkage Estimation}
Shrinking a noisy mean towards a fixed target is a classical idea in statistics, running from Stein's inadmissibility result~\cite{stein1956} through the James--Stein estimator~\cite{james1961estimation} to its empirical-Bayes reading and positive-part variant~\cite{efron1973stein,lehmann1998theory}. In classification it appears as nearest shrunken centroids~\cite{tibshirani2002nsc} and, in few-shot learning, as prototype classifiers~\cite{snell2017protonet} with SimpleShot's centre-then-normalise transform~\cite{wang2019simpleshot}; our NCM baseline and centred variant are the VLM analogues of that line, except that the vector we subtract is a \emph{cross-modal} offset rather than a within-modality mean. Ledoit--Wolf covariance shrinkage~\cite{ledoit2004well} appears inside the GDA baseline~\cite{wang2024gda}, but it shrinks the \emph{covariance}, not the prior--evidence ratio at issue here. In the VLM literature, Goswami \emph{et al.}~\cite{goswami2026protomix} recently derived the bias--variance decomposition of the blended prototype and observed empirically that the optimal mixing weight falls as shots decrease, but selected it by grid search against oracle statistics and, in their own words, still ``require a validation set to tune the prototype mixing coefficient.'' Sec.~\ref{sec:theory} completes that derivation and identifies the plug-in with positive-part James--Stein shrinkage. We claim no novelty for the estimator itself~\cite{lehmann1998theory}, nor for the classical fact that bias and variance do not combine additively under $0/1$ loss~\cite{friedman1997bias,domingos2000unified}; what is new is the vision--language instantiation and its cause, together with the identification, its corollaries, and the measurement of the resulting cost. The class-independent component that breaks the MSE argument is the \emph{modality gap} of Liang \emph{et al.}~\cite{liang2022mind}; we appear to be the first to quantify its decision-time irrelevance for blended prototypes. Wider evidence agrees that the gap can be manipulated without tracking accuracy~\cite{zhou2023dn,liang2022mind,schrodi2025two}, and the one few-shot entry uses a \emph{learned} map rather than a subtraction~\cite{yang2024crossmodal}; our centred variant (Sec.~\ref{sec:results}) adds the missing few-shot data point. Weight-space interpolation~\cite{wortsman2022wiseft} is a cousin of feature-space blending; we analyse the latter, which admits a closed form.

\section{The Blending Ratio in Theory}\label{sec:theory}

We take the bias--variance justification of blending at face value, push it to its closed form (Sec.~\ref{sec:mseopt}), quantify its estimation precision from $K$ shots (Sec.~\ref{sec:sampling}), then prove why this approach to optimisation was never the right goal (Sec.~\ref{sec:whywrong}). The algebra is elementary; Sec.~\ref{sec:results} scores every prediction derived here.

\subsection{Setup}
Fix a class $c$ and drop the subscript. Let $\bm x_1,\dots,\bm x_K\in\mathbb R^d$ be the $\ell_2$-normalised image features of the $K$ support shots, drawn from a population with mean $\mustar=\mathbb E[\bm x]$ and covariance $\bm\Sigma^*$. The empirical prototype $\muhat=\frac1K\sum_i\bm x_i$ is unbiased with $\operatorname{Cov}[\muhat]=\bm\Sigma^*/K$. The text prototype $\vt$, from a fixed prompt, is deterministic. The blended prototype is
\begin{equation}
\bm p(\lambda)\;=\;\lambda\,\muhat+(1-\lambda)\,\vt,\qquad \lambda\in[0,1],
\label{eq:blend}
\end{equation}
and classification is by cosine similarity to the re-normalised $\bm p_c(\lambda)$ over classes $c=1,\dots,C$. Two scalars control the theory below:
\begin{equation}
g^2 \;=\; \bigl\|\vt-\mustar\bigr\|^2,
\qquad
v \;=\; \operatorname{tr}(\bm\Sigma^*)/K .
\label{eq:gv}
\end{equation}
Where a sampling distribution is needed we invoke the isotropic Gaussian model
\begin{equation}
\text{(A1)}\qquad \bm x_i\overset{\text{iid}}{\sim}\mathcal N\bigl(\mustar,\sigma^2\bm I_d\bigr),\qquad \sigma^2=v K/d ,
\label{eq:a1}
\end{equation}
and we flag each use; results not citing (A1) are assumption-free.

\subsection{The MSE-Optimal Ratio and Its James--Stein Plug-In}\label{sec:mseopt}
Because $\muhat-\mustar$ has mean zero and $\vt-\mustar$ is deterministic, the mean-squared error of the blend decomposes as
\begin{equation}
\operatorname{MSE}(\lambda)\;=\;\mathbb E\bigl\|\bm p(\lambda)-\mustar\bigr\|^2\;=\;(1-\lambda)^2 g^2+\lambda^2 v ,
\label{eq:mse}
\end{equation}
a decomposition already published by Goswami \emph{et al.}~\cite{goswami2026protomix}, who use it to argue that blending acts as a shrinkage estimator.

\begin{proposition}[MSE-optimal ratio]\label{prop:lamstar}
$\operatorname{MSE}(\lambda)$ is strictly convex with unique minimiser and value
\begin{equation}
\lambda^{*}\;=\;\frac{g^{2}}{g^{2}+v},
\qquad
\operatorname{MSE}(\lambda^{*})\;=\;\frac{g^{2}v}{g^{2}+v}\;=\;\Bigl(\tfrac{1}{g^2}+\tfrac{1}{v}\Bigr)^{-1}.
\label{eq:lamstar}
\end{equation}
\end{proposition}

The optimal ratio is the fraction of total uncertainty attributable to the prior's bias; the optimal risk is the harmonic combination of $g^2$ and $v$, an inverse-variance weighting, so blending treats the text prototype as a second, biased measurement of the class mean, the reading we test. Since $\operatorname{MSE}(\lambda^*)\le\min(g^2,v)$, the blend never loses to the better endpoint \emph{in MSE}, and Eq.~\eqref{eq:lamstar} derives, rather than observes, the falling-weight trend of~\cite{goswami2026protomix}.

Eq.~\eqref{eq:lamstar} involves the unknown $\mustar,\bm\Sigma^*$; the natural plug-ins are $\hat d^2=\|\vt-\muhat\|^2$ and $\hat v=\operatorname{tr}(\hat{\bm\Sigma})/K$ from the unbiased within-class scatter ($K\ge2$).

\begin{proposition}[Plug-in, and its James--Stein form]\label{prop:jshat}
$\mathbb E[\hat v]=v$, while the naive distance overshoots,
\begin{equation}
\mathbb E\bigl[\hat d^{2}\bigr]=g^{2}+v ,
\label{eq:bias}
\end{equation}
so $\hat g^2=\hat d^2-\hat v$ is unbiased for $g^2$. Substituting $(\hat g^2,\hat v)$ into Eq.~\eqref{eq:lamstar} and clipping to $[0,1]$ yields
\begin{equation}
\lhat\;=\;\Bigl(1-\hat v\big/\hat d^{2}\Bigr)_{+},
\label{eq:jshat}
\end{equation}
the positive-part James--Stein coefficient with the text prototype as shrinkage target. \emph{(Proofs of Propositions~\ref{prop:lamstar}--\ref{prop:jshat} in the supplementary material.)}
\end{proposition}

\looseness=-1 The identification is with the classical estimator~\cite{james1961estimation,efron1973stein,lehmann1998theory}: for $\muhat\sim\mathcal N(\mustar,\tau^2\bm I_d)$ shrinking towards a fixed point $\vt$,
\begin{equation}
\muhat_{\text{JS}}=\vt+\Bigl(1-\tfrac{(d-2)\,\tau^{2}}{\|\muhat-\vt\|^{2}}\Bigr)_{+}\!\bigl(\muhat-\vt\bigr),
\label{eq:jsclassic}
\end{equation}
i.e.\ a blend with ratio $(1-(d-2)\tau^2/\hat d^2)_+$, $\tau^2$ the per-coordinate variance of the \emph{prototype} (under (A1), $\tau^2=v/d$), so Eqs.~\eqref{eq:jshat} and~\eqref{eq:jsclassic} differ by $\frac{d-2}{d}=0.996$ at $d{=}512$: immaterial. \emph{(a)} $\lhat$ is computable per class from the support set alone. \emph{(b)} $\hat d^2\le\hat v\Rightarrow\lhat=0$ is an a-priori criterion for ignoring few-shot evidence, abstaining when the text--image distance is explicable by sampling noise alone; no grid search possesses such a property. \emph{(c)} If the bias--variance story is right, $\lhat$ should also predict the family's tuned trust hyperparameters. All three are registered as falsifiable predictions; (c) will fail, instructively. The honest attribution: Eq.~\eqref{eq:mse} is prior work, and Eqs.~\eqref{eq:lamstar}--\eqref{eq:jshat} are one line of calculus plus a textbook identity, which nevertheless comes with a certificate.

\begin{proposition}[Dominance at the family's own objective]\label{prop:dom}
Under (A1), if $K\ge2$ and $(K-1)(d-4)\ge2$, then for \emph{every} $(\mustar,\vt)$,
$\mathbb E\|\bm p(\lhat)-\mustar\|^2<\mathbb E\|\muhat-\mustar\|^2$.
\end{proposition}
\begin{IEEEproof}[Proof sketch]
$\lhat$ is Baranchik's rule with $c=\tfrac dm$, $m=d(K-1)$, and $c\le\tfrac{2(d-2)}{m+2}$ is exactly $(K-1)(d-4)\ge2$~\cite{baranchik1970family,lehmann1998theory}; the positive part only improves. Full proof in the supplement.
\end{IEEEproof}

What is new is the identification, together with (a) this certificate, (b) the sampling theory below, and (c) the proof that the objective is wrong: the family's own criterion uniformly certifies the estimator that Sec.~\ref{sec:results} shows trailing the accuracy oracle by $8.5$ points.

\subsection{How Precise Is the Plug-In? Sampling Theory}\label{sec:sampling}
Eq.~\eqref{eq:jshat} will be shown to fail by more than half the unit interval, a failure that indicts the \emph{objective} only if the \emph{estimator} is precise; under (A1) its error budget is exactly quantifiable.

\begin{proposition}[Sampling distribution of $\lhat$]\label{prop:sampling}
Under (A1), $\hat d^2$ and $\hat v$ are independent, with
\begin{equation}
\hat d^{2}\sim\tfrac{v}{d}\,\chi^{2}_{d}\bigl(\tfrac{g^{2}d}{v}\bigr),
\quad
\hat v\sim\tfrac{v\,\chi^{2}_{d(K-1)}}{d(K-1)},
\label{eq:chidists}
\end{equation}
hence $\operatorname{Var}[\hat d^{2}]=\tfrac{2v}{d}(v+2g^{2})$ and $\operatorname{Var}[\hat v]=\tfrac{2v^{2}}{d(K-1)}$. To first order, $\mathbb E[\lhat]=\lambda^*$, and the delta method gives
\begin{equation}
\operatorname{SE}\bigl[\lhat\bigr]\;=\;(1-\lambda^{*})\sqrt{\frac{2}{d}}\,\sqrt{\frac{1}{K-1}+1-\lambda^{*2}}\;+\;O\!\bigl(d^{-1}\bigr).
\label{eq:se}
\end{equation}
\end{proposition}
\emph{(Proof in the supplementary material.)}

At this scale ($d{=}512$, measured mean $\lhat\approx0.96$), Eq.~\eqref{eq:se} gives $\operatorname{SE}\approx9.5\times10^{-4}$ at $K{=}16$ and $2.6\times10^{-3}$ at $K{=}2$. But (A1) has no support on the unit sphere, and contrastive features are strongly anisotropic, so we do not stop at the isotropic figure.

\begin{proposition}[Sampling law under arbitrary covariance]\label{prop:sampgen}
For $\bm x_i\overset{\text{iid}}{\sim}\mathcal N(\mustar,\bm\Sigma^*)$ with any $\bm\Sigma^*\succeq0$, $\hat d^2$ and $\hat v$ remain independent, with $\bm b=\vt-\mustar$, $\operatorname{Var}[\hat d^{2}]=\tfrac{4}{K}\bm b^{\!\top}\bm\Sigma^*\bm b+\tfrac{2}{K^{2}}\operatorname{tr}(\bm\Sigma^{*2})$ and $\operatorname{Var}[\hat v]=2\operatorname{tr}(\bm\Sigma^{*2})/(K^{2}(K-1))$; the delta method gives, with $\deff=(\operatorname{tr}\bm\Sigma^*)^{2}/\operatorname{tr}(\bm\Sigma^{*2})\in[1,d]$,
\begin{equation}
\operatorname{SE}[\lhat]=(1{-}\lambda^*)\sqrt{\tfrac{2}{(K-1)\deff}+\tfrac{4\bm b^{\!\top}\bm\Sigma^*\bm b/K+2\operatorname{tr}(\bm\Sigma^{*2})/K^2}{(g^2+v)^{2}}}\,,
\label{eq:segen}
\end{equation}
up to the same-order remainder as Eq.~\eqref{eq:se}, to which it reduces exactly when $\bm\Sigma^*=\sigma^2\bm I$.
\end{proposition}

\looseness=-1 The isotropic formula divides by the nominal dimension, but the honest denominator is the spectrum's participation ratio, and the difference is an order of magnitude: estimated per class from pool residuals (unbiased $\operatorname{tr}(\bm\Sigma^{2})$ per~\cite{chen2010twosample}), $\deff$ has median \DEFFMED\ (range \DEFFRANGE) against $d\in\{512,768,1024\}$. The prediction then meets the measurement: the per-class seed-to-seed SD of $\lhat$ across the five recorded support draws is \SEEDSDKtwo\ at $K{=}2$ and \SEEDSDKsixteen\ at $K{=}16$ (medians over classes, then strata), $\AONEOPT\times$ above the (A1) figure at $K{=}2$, and Eq.~\eqref{eq:segen} lands at a measured-to-predicted ratio of \SECLOSEKtwo\ and \SECLOSEKsixteen\ respectively (\texttt{code/probe\_capacity.py}). Either way, the measured $0.63$ gap to the accuracy-optimal ratio stays $26$--$770\times$ above any sampling spread: objective error, not estimation error. The same concentration warns the family at large: a trust parameter set by an MSE-type moment rule is set \emph{reproducibly}, and consistency is easily mistaken for correctness. All formulas here are confirmed by simulation, and every statistical component was validated on synthetic ground truth before touching real features, where the per-class $\lhat$ tracks the pool-estimated $\lambda^*$ (Spearman $\rho\approx0.59$; protocol in the supplementary material).

\subsection{Why Minimum MSE Is the Wrong Target}\label{sec:whywrong}
One mismatch is normalisation, since cosine classification is invariant to a common rescaling of all prototypes while MSE penalises length and direction alike. The decisive one is geometry, where the decision rule compares classes, so any component of prototype error \emph{shared across classes} is irrelevant at decision time. Write the text prototypes as
\begin{equation}
\vt_c=\mustar_c+\bm\Delta+\bm\delta_c ,\qquad \textstyle\sum_c\bm\delta_c=\bm0,
\label{eq:decomp}
\end{equation}
where $\bm\Delta$ is the class-independent offset (the mean of $\vt_c-\mustar_c$ over classes) and $\bm\delta_c$ the class-specific remainder.

\begin{theorem}[Common shifts are (almost) decision-free]\label{thm:shift}
Let every prototype be shifted by the same vector, $\bm p_c\mapsto\bm p_c+\bm\Delta$.
\emph{(i)} For inner-product scoring $s_c(\bm x)=\langle\bm x,\bm p_c\rangle$, all pairwise score differences are unchanged for every input: predictions are identical.
\emph{(ii)} For nearest-prototype scoring $s_c(\bm x)=-\|\bm x-\bm p_c\|^2$, the shifted classifier equals the original classifier applied to $\bm x-\bm\Delta$; decision regions translate rigidly. Consequently the induced risk depends on $\bm\Delta$ only through its projection $\Pi_{\mathcal W}\bm\Delta$ onto
\begin{equation}
\mathcal W=\operatorname{span}\{\bm w_{cc'}\},\quad \bm w_{cc'}=\bm p_c-\bm p_{c'},\quad \dim\mathcal W\le C-1 .
\label{eq:proj}
\end{equation}
\emph{(iii)} If $\bm\Delta$ is isotropically distributed (or in generic position) relative to $\mathcal W$, then for each pair
\begin{equation}
\mathbb E\bigl\langle\bm\Delta,\bm w_{cc'}\bigr\rangle^{2}=\frac{\|\bm\Delta\|^{2}\,\|\bm w_{cc'}\|^{2}}{d},
\label{eq:iso}
\end{equation}
so boundary displacements scale as $\|\bm\Delta\|/\sqrt d$, while Eq.~\eqref{eq:mse} charges the blend $(1-\lambda)^2\|\bm\Delta\|^2$ in full: the MSE objective overweights the common offset by a factor of order $d$. The premise is tested: the squared cosine between $\widehat{\bm\Delta}$ and the class-difference directions averages $1.0$--$3.3\times10^{-3}$ across the eight datasets with official train splits (the two self-split datasets, Sec.~\ref{sec:datasets}, are not covered by this probe), against $1/d=1.95\times10^{-3}$: generic position within a factor of two.
\end{theorem}
\emph{(Proof in the supplementary material.)}

\looseness=-1 Cosine scoring is inner-product scoring after per-class renormalisation; the invariance of (i) is therefore not an identity, and with $\|\widehat{\bm\Delta}\|=1.00$ on unit-norm prototypes there is no small parameter to expand in. We therefore measure it: translating every prototype by $a\widehat{\bm\Delta}$ at the oracle ratio leaves accuracy essentially unchanged for $|a|\le\tfrac14$ and degrades gracefully thereafter (Sec.~\ref{sec:perclass}), whereas Eq.~\eqref{eq:mse} charges the full $(1-\lambda)^2\|\bm\Delta\|^2$: the invariance is approximate but operative over exactly the range the ratio moves the prototypes through. In contrastive vision--language spaces the modality gap~\cite{liang2022mind} (text and image embeddings in separated cones) is by nature largely class-independent. We quantify it as the \emph{class-independent share} of the squared prototype distance, $1-\overline{\|\bm\delta_c\|^2}/\overline{\|\vt_c-\mustar_c\|^2}$ with class means estimated from held-out pool data: $78.3\%$ \CI{74.6}{82.1} across all 4{,}800 cells, from $73.9\%$ (ResNet-50) to $86.6\%$ on SigLIP, whose sigmoid objective differs from CLIP's InfoNCE yet produces an even more class-independent gap.

For two classes the \emph{right} objective has an exact answer.

\begin{corollary}[Two-class risk, exact]\label{cor:tworisk}
Let $C{=}2$, $\bm x\,|\,\pm\sim\mathcal N(\mustar_\pm,\sigma^2\bm I_d)$, $\bm u=\mustar_+-\mustar_-$, and $\vt_\pm=\mustar_\pm+\bm\Delta\pm\bm\delta$ per Eq.~\eqref{eq:decomp}. With exact image prototypes, the blend classifier $\operatorname{sign}\langle\bm x-\tfrac{\bm p_++\bm p_-}{2},\bm p_+-\bm p_-\rangle$ at $s=1-\lambda$ has balanced error
\begin{equation}
\begin{aligned}
\operatorname{err}(s)&=\tfrac12\textstyle\sum_{\pm}\Phi\!\Bigl(-\,A_{\pm}(s)\big/\sqrt{B(s)}\Bigr),\\
A_{\pm}(s)&=\tfrac12\|\bm u\|^{2}+(\gamma\mp\beta)\,s\mp2\eta s^{2},\\
B(s)&=\sigma^{2}\bigl(\|\bm u\|^{2}+4\gamma s+4\|\bm\delta\|^{2}s^{2}\bigr),
\end{aligned}
\label{eq:errexact}
\end{equation}
with $\gamma=\langle\bm u,\bm\delta\rangle$, $\beta=\langle\bm\Delta,\bm u\rangle$ and $\eta=\langle\bm\Delta,\bm\delta\rangle$. The offset $\bm\Delta$ enters only through the two scalars $\beta$ and $\eta$, its projections onto the class difference and onto the class-specific residual (i.e.\ onto the decision normal $\bm w=\bm u+2s\bm\delta$ of Theorem~\ref{thm:shift}), whereas $g^2$ in Eq.~\eqref{eq:lamstar} contains $\|\bm\Delta\|^2+\|\bm\delta\|^2$ in full. On real prototypes $\eta$ is small but not zero (mean cosine $0.010$ over the same eight datasets), kept because dropping it would make ``exact'' false. Moreover, for $\beta=\eta=0$, $\operatorname{err}$ is non-decreasing in $s$ (Cauchy--Schwarz): without estimation noise the risk-optimal blend is $\lambda=1$, however large the ``bias'' $g^2$.
\end{corollary}

Risk shrinks toward the text prototype only because the image prototype is \emph{estimated}. Reinstating $K$-shot means $\muhat_\pm=\mustar_\pm+\bm\varepsilon_\pm$, $\bm\varepsilon_\pm\sim\mathcal N(\bm 0,\tfrac vd\bm I_d)$, and maximising the deflection ratio (expected margin over root total score variance) replaces $B$ with a noise-aware quadratic $B_v$, to $O(v/d)$ (supplementary material); the stationarity condition $2A'B_v=AB_v'$, generically cubic, loses its quadratic terms, leaving a closed form:

\begin{corollary}[Risk-optimal shrinkage, noise-aware]\label{cor:sstar}
For $\beta=\eta=0$, the deflection-optimal amount of shrinkage toward the text prototypes is
\begin{equation}
s^{*}\;=\;\frac{2v\bigl(\|\bm u\|^{2}+2\gamma\bigr)}
{4\bigl(\|\bm u\|^{2}\|\bm\delta\|^{2}-\gamma^{2}\bigr)+4\gamma v+2v\|\bm u\|^{2}}\; ,
\label{eq:sstar}
\end{equation}
clipped to $[0,1]$. Three properties contradict Eq.~\eqref{eq:lamstar} point for point: $s^*=O(v)$ as $v\to0$ (shrinkage is bought by estimation noise alone); $s^*$ is independent of $\bm\Delta$ and of $g^2$ (the modality gap has no business in the formula); and as $\bm\delta\to\bm0$, $s^*\to1$: risk says \emph{use the text prototypes outright}, while $\lambda^*_{\textsc{mse}}\to1$ demands the opposite ever more strongly as the (irrelevant) gap grows.
\end{corollary}

\looseness=-1 Eqs.~\eqref{eq:errexact}--\eqref{eq:sstar} are verified by Monte Carlo in the released script (risk formula to $<10^{-3}$; Eq.~\eqref{eq:sstar} against the numerical argmin). The corollaries are offered as the precise mechanism of the failure, not a new estimator: $\bm u$ and $\bm\delta$ are exactly what $K$ shots cannot pin down. The prediction: $\lhat$ will sit near $1$ (measured mean $0.96$) while the accuracy-optimal ratio sits far lower ($0.33$); the gap should scale across encoders with the class-independent share of Table~\ref{tab:perbackbone} (Sec.~\ref{sec:backbones}: it fails); subtracting an estimated $\bm\Delta$ inside Eq.~\eqref{eq:jshat} repairs direction, not magnitude. Everything here was written down before the full matrix ran; Sec.~\ref{sec:results} scores the predictions as stated, including the two that failed.

\subsection{The Family's Ceiling Is Geometric}\label{sec:ceiling}
The corollaries explain why the MSE target picks the wrong ratio. What they do not yet give is a bound on what the \emph{best} ratio could achieve. Drop isotropy: $\bm x\,|\,c\sim\mathcal N(\mustar_c,\bm\Sigma)$, $\bm\Sigma\succ0$ arbitrary and shared. For a class pair, $\bm u$ is the mean difference, $\bm\delta$ the class-specific half-difference of the text offsets (Eq.~\eqref{eq:decomp}); every blend scores the pair along $\bm w(s)=\bm u+2s\bm\delta$, $s=1-\lambda$, while a probe may score along any $\bm w\in\mathbb R^d$.

\begin{theorem}[Capacity of the blending family]\label{thm:capacity}
With balanced priors and each direction's optimal threshold, a linear rule with direction $\bm w$ errs with probability $\Phi(-m(\bm w)/2)$, $m(\bm w)=\bm u^{\!\top}\bm w/\sqrt{\bm w^{\!\top}\bm\Sigma\bm w}$. Let $a_1=\bm u^{\!\top}\bm u$, $a_2=\bm u^{\!\top}\bm\delta$ and $q_{11},q_{12},q_{22}$ the corresponding $\bm\Sigma$-quadratic forms. Then:
\emph{(i)} $m^2$ has a unique stationary point on the path, $s^{\star}=\frac{a_1q_{12}-a_2q_{11}}{2\,(a_2q_{12}-a_1q_{22})}$, and its maximum over $s\in[0,1]$ is attained at $s^\star$ or an endpoint;
\emph{(ii)} the attainable separations nest,
\begin{equation}
\max_{s}\,m(\bm w(s))^{2}\;\le\!\!\max_{\bm w\in\operatorname{span}\{\bm u,\bm\delta\}}\!\!m(\bm w)^{2}\;\le\;\bm u^{\!\top}\bm\Sigma^{-1}\bm u,
\label{eq:nest}
\end{equation}
with equality on the right iff $\bm\Sigma^{-1}\bm u\in\operatorname{span}\{\bm u,\bm\delta\}$;
\emph{(iii)} if $\bm\Sigma=\sigma^2\bm I$ then $s^{\star}=0$.
\end{theorem}

\looseness=-1 Part \emph{(iii)} is the sharpest form of the objective mismatch: with exact prototypes, the risk-optimal shrinkage is zero unless the covariance has \emph{shape} (the trace is all Eq.~\eqref{eq:lamstar} can see of $\bm\Sigma$), while both \emph{(i)} and the ceiling \emph{(ii)} are functions of the whitened geometry. Part \emph{(ii)} is the capacity statement behind Sec.~\ref{sec:results}, where the blend commands one direction on a one-parameter path and a probe commands all of $\mathbb R^d$.

\looseness=-1 The theorem is a two-class Gaussian idealisation, so we measure its content rather than assume it. With $\bm\Sigma$ from pool residuals (OAS-regularised~\cite{chen2010shrinkage}), the blend path attains on average \CAPRATIO\ of the unrestricted Mahalanobis separation over all class pairs of all $50$ strata; converting Eq.~\eqref{eq:nest} to error gives each stratum a predicted \emph{headroom}, and predicted headroom meets the measured CLAP-minus-oracle surplus at Spearman \CAPSPEAR\ \CAPSPEARCI\ (cluster bootstrap over datasets), \CAPPARTIAL\ after rank-partialling out raw difficulty, medians \CAPWDS\ within datasets and \CAPWBK\ within backbones. Proofs and the per-stratum table are in the supplement; \texttt{code/probe\_capacity.py} regenerates everything.

\section{Estimating the Ratio Without Validation Data}\label{sec:estimation}

If MSE shrinkage sets $\lambda$ by the wrong criterion, set it by accuracy instead, using the only labelled data a few-shot learner legitimately owns: the $K$ support shots.

\subsection{A Leave-One-Out Ratio}
For each candidate $\lambda$ on a grid ($101$ values in $[0,1]$), we score every support sample $\bm x_i$ of class $y_i$ against prototypes in which the sample's own class uses the deflated mean $\muhat^{(-i)}_{y_i}=(K\muhat_{y_i}-\bm x_i)/(K-1)$, other classes their full prototypes, and select the $\lambda$ maximising leave-one-out (LOO) accuracy over all $CK$ samples (Algorithm~\ref{alg:loo}). Unlike HOSO~\cite{vorster2026hoso}, which spends one shot per class as a held-out cache, this uses every shot for both fitting and validation; it requires $K\ge2$. Cross-validation on few shots has been criticised as unreliable~\cite{luo2026fewtrans}, chiefly because splitting changes the effective shot count. That critique bites less here: only a single scalar is selected, and the shot-count shift is analytically correctable. By Eq.~\eqref{eq:lamstar}, $\lambda$ corresponds to an odds ratio $r=\lambda/(1-\lambda)=g^2/v$ with $v\propto1/K$, so a ratio fitted at effective count $K{-}1$ is debiased to $K$ by
\begin{equation}
r\;\longmapsto\;r\cdot\frac{K}{K-1},
\qquad
\lhat_K=\frac{r'}{1+r'} .
\label{eq:kcorr}
\end{equation}
This is the one place the MSE theory, wrong about the \emph{level} of $\lambda$, earns its keep: it correctly predicts the $K$-\emph{dependence}. Empirically the correction is directionally right but small ($+0.1$ to $+0.5$\,pp at low $K$), so the plain LOO ratio is primary. Selecting on the support set does not overfit it, by capacity: one scalar from a grid of $101$ exerts at most $\log 101$ nats of overfitting pressure against $CK$ leave-one-out votes. The regime to distrust is small $C$ \emph{and} small $K$ (EuroSAT at $K{=}2$: twenty votes), where the LOO ratio is noisiest; even there it never underperforms the zero-shot fallback.

\begin{algorithm}[t]
\caption{Leave-one-out blending ratio (support set only)}
\label{alg:loo}
\begin{algorithmic}[1]
\REQUIRE support features $\bm X\!\in\!\mathbb R^{N\times d}$ ($N{=}CK$, $\ell_2$-normalised), labels $\bm y$, text prototypes $\bm T\!\in\!\mathbb R^{C\times d}$, grid $\Lambda$
\STATE $\muhat_c \leftarrow$ class means; $\muhat^{(-i)}\!\leftarrow\!(K\muhat_{y_i}-\bm x_i)/(K{-}1)$ for all $i$
\STATE precompute $\bm A{=}\bm X\hat{\bm M}^{\!\top}$, $\bm B{=}\bm X\bm T^{\!\top}$ and the six Gram scalars of Eq.~\eqref{eq:quad} (cross-class and own-class)
\FOR{$\lambda \in \Lambda$}
  \STATE form scores $\bm S(\lambda)$ by Eq.~\eqref{eq:quad}; overwrite own-class entries $S_{i,y_i}$ with their leave-one-out counterparts
  \STATE $a(\lambda)\leftarrow\frac1N\sum_i \mathbf 1[\arg\max_c S_{ic}(\lambda)=y_i]$
\ENDFOR
\RETURN $\lhat_{\text{LOO}}=\arg\max_{\lambda\in\Lambda} a(\lambda)$; optionally apply Eq.~\eqref{eq:kcorr}
\end{algorithmic}
\end{algorithm}

Every blend interpolates the \emph{same} two prototype sets, so with $\bm A=\bm X\hat{\bm M}^{\!\top}$, $\bm B=\bm X\bm T^{\!\top}$ precomputed once, the cosine score of the re-normalised blend is, exactly,
\begin{equation}
S_{ic}(\lambda)=\frac{\lambda A_{ic}+(1-\lambda)B_{ic}}
{\sqrt{\lambda^2\|\muhat_c\|^2+2\lambda(1-\lambda)\langle\muhat_c,\vt_c\rangle+(1-\lambda)^2\|\vt_c\|^2}},
\label{eq:quad}
\end{equation}
an $O(nC)$ elementwise update per grid point---a thousandfold reduction, verified bit-identical to the naive implementation---driving Algorithm~\ref{alg:loo} and the oracle sweep alike.

\emph{Why a global ratio?} Nothing prevents running Algorithm~\ref{alg:loo} per class, and Eq.~\eqref{eq:jshat} is per-class by construction. But selecting $C$ ratios from $CK$ leave-one-out predictions gives each class only $K$ votes for its own knob, two at $K{=}2$, and selection noise overwhelms signal; per-class LOO underperforms the global ratio everywhere we tried it. Pooling across classes, the empirical-Bayes route, is the natural repair, left to future work (Sec.~\ref{sec:discussion}).

\subsection{Making GDA Validation-Free}
GDA~\cite{wang2024gda} is described as training-free, but its official implementation grid-searches the weight $\alpha$ blending generative and zero-shot logits on a labelled validation set. We supply $\alpha$ by the same leave-one-out rule (the own-class Gaussian mean is refitted for each held-out sample, while the shared covariance, whose dependence on any single sample is $O(1/N)$, is kept fixed), yielding a validation-free GDA that closes about half the gap to the validation-tuned version at $K\in\{2,4\}$ and none of it at $K\ge8$ (Table~\ref{tab:leaderboard}).

\subsection{Reference Points That Use Forbidden Data}
The quantities below anchor the analysis from above, always as bounds, never as methods. The \emph{accuracy-oracle ratio} grid-searches a single global $\lambda$ on the test labels---the best the family's contested hyperparameter can possibly deliver. The \emph{MSE-oracle ratio} evaluates Eq.~\eqref{eq:lamstar} per class with $\mustar,\bm\Sigma^*$ estimated from pool data outside the support set, isolating objective error from estimation error. And \emph{Tip-Adapter} with $\alpha,\beta$ tuned on the test set reproduces the family's historical practice~\cite{zhang2022tip}. The accuracy oracle is deliberately global: a per-class oracle would fit $C$ free parameters to the test labels (overfitting, not a bound), and Sec.~\ref{sec:perclass} measures exactly how much of its apparent gain that is.

\section{Benchmark Protocol}\label{sec:protocol}

\subsection{Datasets and Splits}\label{sec:datasets}
We use ten classification datasets (Table~\ref{tab:datasets}): FGVC-Aircraft~\cite{maji2013aircraft}, Caltech101~\cite{fei2004caltech}, Stanford Cars~\cite{krause2013cars}, DTD~\cite{cimpoi2014dtd}, EuroSAT~\cite{helber2019eurosat}, Flowers102~\cite{nilsback2008flowers}, Food-101~\cite{bossard2014food}, ImageNet~\cite{deng2009imagenet}, Oxford-IIIT Pets~\cite{parkhi2012pets}, and SUN397~\cite{xiao2010sun}. Images, class names, and zero-shot prompt templates come from the publicly pinned \mbox{clip-benchmark} webdataset distributions, whose zero-shot numbers our sanity gate (below) verifies against that project's public results. Deviations from the CoOp-lineage protocol~\cite{zhou2022coop}: (i) splits are clip-benchmark's, not CoOp's \texttt{split\_zhou} files, so absolute numbers are not cell-comparable with that lineage---every method here sees identical splits, as our claims require; (ii) ImageNet and SUN397 ship only evaluation shards, so we form a deterministic per-class 68/32 pool/test split (seeded once, before any experiment ran); (iii) the clip-benchmark Caltech101 is the 102-class VTAB variant whose literal \texttt{background} class no prompt can name, placing zero-shot ${\sim}4$ points below 101-class numbers (a protocol difference, not an extraction bug); (iv) Flowers102's pool holds ten images per class, capping it at $K\le4$; (v) UCF101 has no clip-benchmark distribution and is omitted.

\emph{Support sampling.} For each (dataset, backbone, $K$, seed) cell, $K$ shots per class are drawn from the pool with one of five seeds; all methods in a cell see the identical support set, making every comparison paired. Population statistics for the MSE-oracle come from the pool \emph{minus} the support set, so oracle and estimate never share samples. Test data is touched only by final evaluation and the explicitly-labelled oracles.

\subsection{Backbones and Text Priors}
Five backbones: OpenAI CLIP RN50, ViT-B/32, ViT-B/16, ViT-L/14, and SigLIP ViT-B/16~\cite{zhai2023siglip} via OpenCLIP~\cite{cherti2023openclip}, all frozen, features $\ell_2$-normalised once and cached; SigLIP's different pre-training objective tests whether the diagnosis is CLIP-specific. Four text-prior tiers per dataset: the generic \emph{photo} template (``a photo of a \{\}.''), the \emph{dataset-specific} template, the full clip-benchmark prompt \emph{ensemble}, and \emph{CuPL}-style LLM descriptions~\cite{pratt2023cupl,menon2023dclip} generated with \mbox{gpt-4.1-mini} (three query templates for each of the 2{,}092 classes; prompts and outputs released), cleaned, class-name-prefixed when absent, encoded per sentence, and averaged into a unit-norm prototype, applied uniformly across datasets and text encoders. The four tiers instrument the prior-quality axis of Table~\ref{tab:priors}.

\subsection{Methods}
Fifteen classifiers are evaluated on the dataset-specific tier. Eleven are validation-free (zero-shot, NCM, four blends, GDA at $\alpha{=}1$ and by LOO, HOSO, CLAP, LP++), and four consume data a deployed system would not have (Sec.~\ref{sec:estimation}: the MSE-oracle and accuracy-oracle ratios, Tip-Adapter's test-tuned $(\alpha,\beta)$, and GDA's validation-tuned $\alpha$); the released records log, per cell, which quantity each method selected and from which split.
Validation-free methods: zero-shot; NCM (image prototypes); the James--Stein blend of Eq.~\eqref{eq:jshat} and its $\bm\Delta$-centred variant; the LOO blend with and without $K$-correction; GDA with $\alpha{=}1$ and with $\alpha$ by LOO; and three published baselines---CLAP~\cite{silva2024closer}, LP++~\cite{huang2024lppp}, HOSO~\cite{vorster2026hoso}---re-implemented on cached features from their official code line by line (HOSO from its paper; no code exists). Our GDA and LP++ agree with the official implementations to relative error $10^{-8}$ on identical inputs. Unavoidable deviations, disclosed: no baseline uses image augmentation here (features are cached once; official CLAP averages 20 augmented views, GDA uses 10 augmentation epochs; this hurts GDA most, whose covariance degenerates at $K{=}1$, where we report it as zero-shot); CLAP's anchor scale, unrecoverable from normalised caches, is fixed at its typical value $10$; LP++'s optional early-epoch selection on an extra validation split is disabled (its validation-free path, the paper's stated intent).

\subsection{Statistics and the Combination Guard}\label{sec:stats}
Every accuracy difference is computed \emph{within} cells (identical support sets) and aggregated by a cluster bootstrap whose \textbf{resampling unit is the dataset} (10{,}000 resamples). Cells sharing a dataset share a test split, pool, vocabulary and text prior; resampling the 4{,}800 cells individually would shrink every interval by roughly an order of magnitude. The intervals therefore answer ``would this survive another draw of \emph{datasets}''; $n$ counts cells, the effective sample size is the ten clusters. Ten is few enough that the percentile interval cannot be assumed honest, so we measured it: a variance-matched simulation puts nominal-$95\%$ percentile coverage at $89.3\%$. We therefore report Student-$t$ intervals on the bootstrap standard error ($G-1=9$ df), which the same simulation puts at $93.7$--$94.8\%$; every interval here is the wider, corrected one (\texttt{code/boot\_coverage.py}). No multiplicity correction is applied across the roughly thirty intervals reported; discount lower bounds within a few tenths of zero accordingly---the two closest calls are HOSO's $K{\ge}4$ margin (Table~\ref{tab:headline}, lower bound $+0.01$) and CLAP's SigLIP margin (Table~\ref{tab:perbackbone}, $+0.11$). By-$K$ aggregates are restricted to the 45 dataset$\times$backbone combinations present at every $K$; without this guard, dropped cells masquerade as trends. Cells whose per-class pool residue after support sampling falls below five are dropped and counted (200 of 5{,}000, all Flowers102 at $K\in\{8,16\}$).

\subsection{Verification, Audit Trail, and Reproducibility}
\emph{Sanity gate:} ViT-B/16 prompt-ensemble zero-shot accuracy must land inside a coarse band of published CLIP results before a feature cache is trusted; all ten datasets pass (Table~\ref{tab:datasets}), and the gate caught a mis-specified activation (QuickGELU). \emph{Precision:} eight dataset$\times$backbone shards were computed in both float64 and float32 (the precision of the official baselines, used for the main matrix); every method except HOSO is precision-exact, and HOSO moves by up to $3.1$\,pp (mean $0.44$), so its results carry a ${\sim}0.5$\,pp band (Sec.~\ref{sec:robust}). Choices that could have gone the other way are stated: the pool-residue threshold of five replaced ten solely because ten would silently delete Flowers102, and the failed predictions (corollary~(c), the centred-JS repair, the $K$-correction) are reported with their outcomes. We release the pipeline, the cached features, and the complete per-cell record set: $4{,}800$ cells, carrying all fifteen classifiers on the $1{,}200$ dataset-specific-prompt cells and the eight prompt-independent ones on the other three tiers, with every fitted hyperparameter. Determinism is end-to-end (seeded splits, draws, and initialisations; checksummed caches; per-cell checkpointing), so the matrix extends cell-by-cell while remaining exactly comparable. Total compute, one RTX~4090 workstation: ${\approx}2$ GPU-hours to encode the $1.25$\,M image--backbone pairs plus ${\approx}10$ hours of CPU-bound solver fitting (Sec.~\ref{sec:cost}). Every table and figure regenerates from the released records on a CPU in minutes; every experiment re-runs from the released features without decoding a single image.

\section{Results}\label{sec:results}

\begin{table}[t]
\caption{Benchmark composition (clip-benchmark splits) and the zero-shot sanity gate: ViT-B/16 prompt-ensemble accuracy must land inside a coarse band of published CLIP results before any feature cache is trusted. All ten datasets pass. ImageNet and SUN397 ship only evaluation shards, so a deterministic 68/32 per-class pool/test split is used (Sec.~\ref{sec:datasets}); test sizes are post-split.}
\label{tab:datasets}
\centering\small
\setlength{\tabcolsep}{4pt}
\resizebox{\columnwidth}{!}{%
\begin{tabular}{l ccccc}
\toprule
Dataset & $C$ & $n_{\text{test}}$ & Pool split & ZS ens.\ (\%) & Published\\
\midrule
Aircraft & 100 & 3,333 & official & 24.3 & 20--32 \\
Caltech101 & 102 & 6,085 & official & 82.2 & 80--96 \\
Cars & 196 & 8,041 & official & 64.7 & 60--70 \\
DTD & 47 & 1,880 & official & 44.9 & 40--50 \\
EuroSAT & 10 & 5,400 & official & 55.9 & 36--62 \\
Flowers102 & 102 & 6,149 & official & 71.2 & 65--78 \\
Food101 & 101 & 25,250 & official & 88.7 & 82--92 \\
ImageNet & 1000 & 16,000 & self-split & 68.3 & 64--72 \\
Pets & 37 & 3,669 & official & 89.0 & 85--92 \\
SUN397 & 397 & 34,805 & self-split & 64.4 & 58--70 \\
\bottomrule
\end{tabular}
}
\end{table}

Fig.~\ref{fig:teaser} summarises the three findings; Tables~\ref{tab:perdataset} and~\ref{tab:leaderboard} give full numbers. Every feature cache reproduces published zero-shot behaviour (Table~\ref{tab:datasets}), including the two deliberate protocol quirks, so method differences cannot be extraction artefacts. Differences are paired within cells and quoted as mean [95\% CI] in percentage points (pp); $n$ counts cells.

\subsection{The Theoretically Optimal Ratio Fails}
By Sec.~\ref{sec:theory} the James--Stein blend estimates its own target well (per-class $\lhat$ vs.\ pool-estimated $\lambda^*$: MAE $0.007$), and by Prop.~\ref{prop:dom} it dominates the raw prototype at that target. The classifier it induces nonetheless fails: averaged over all cells it trails the test-set-oracle ratio by $-8.51$ \CI{-12.99}{-4.04} ($n{=}950$ cells in $10$ dataset clusters) and its pooled advantage over zero-shot is $-0.12$ \CI{-8.99}{+8.75}---an average over a sign change, not an equivalence: $10.7$\,pp \emph{below} zero-shot at $K{=}2$ and $6.6$\,pp above at $K{=}16$ (Table~\ref{tab:leaderboard}). The MSE-oracle (Eq.~\eqref{eq:lamstar} at population statistics) confirms objective error: $64.1\%$ mean accuracy against the accuracy-oracle's $72.8\%$ (Table~\ref{tab:leaderboard}, $K\ge2$). The consequence is blunt: at $\lhat\!\approx\!1$ the MSE-optimal rule is, empirically, the nearest-class-mean classifier---beating NCM by $+0.76$ \CI{+0.38}{+1.13}, against NCM's own $-9.27$ \CI{-14.11}{-4.42} to the oracle ratio. \emph{The closed form's entire content is that it discards the text prior}, at the long-known cost in accuracy; what is new is that the family's own derivation leads there. Fig.~\ref{fig:lambda} shows the mechanism: James--Stein ratios crowd the top of the plot (mean $0.96$) wherever the accuracy-optimal ratio lies (mean $0.33$); Fig.~\ref{fig:profile} shows the curve: a broad plateau, then a cliff near $\lambda{=}1$, exactly where Eq.~\eqref{eq:jshat} concentrates.

\begin{figure}[t]
\centering
\includegraphics[width=0.87\columnwidth]{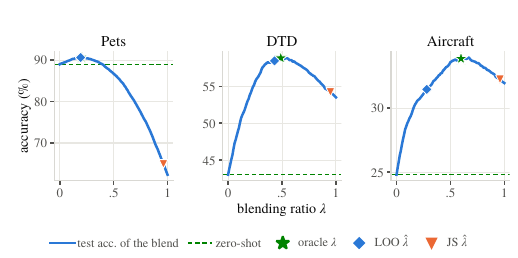}
\caption{Measured risk profiles: test accuracy of the blend vs.\ $\lambda$ (ViT-B/16, $K{=}4$, seed-mean) for a strong (Pets), medium (DTD), and weak (Aircraft) prior. The oracle ratio (star) moves right as the prior weakens; the LOO estimate (diamond) follows from the support set alone, its occasional misses (Aircraft) costing little on the broad plateau; the James--Stein estimate (triangle) sits at the cliff edge on every dataset.}
\label{fig:profile}
\end{figure}

\begin{figure}[t]
\centering
\includegraphics[width=0.80\columnwidth]{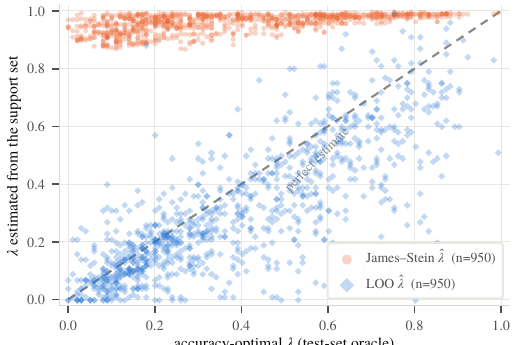}
\caption{Estimated vs.\ accuracy-optimal blending ratio, one point per cell (dataset-specific prompt; per-class estimates averaged for display). The James--Stein estimate saturates near $1$ irrespective of the oracle: a good estimate of the wrong quantity. The LOO estimate follows the diagonal ($\rho=0.79$, MAE $0.12$).}
\label{fig:lambda}
\end{figure}

\looseness=-1 Sec.~\ref{sec:whywrong} predicted both the failure and its signature. Table~\ref{tab:perbackbone} measures the class-independent share of the squared prototype distance at $78.3\%$ \CI{74.6}{82.1} overall: most of the ``bias'' that Eq.~\eqref{eq:mse} attributes to the text prototype is a modality-gap offset~\cite{liang2022mind} that never influences a decision. The effect is strongest on SigLIP ($86.6\%$), whose sigmoid pre-training differs from CLIP's InfoNCE. On strong-prior datasets the damage is severe (Pets: the JS blend loses double digits to zero-shot at low $K$); on weak-prior datasets it merely wastes the prior. The centred variant (subtracting the estimated $\bm\Delta$ before applying Eq.~\eqref{eq:jshat}) lowers the mean ratio to $0.83$ but \emph{worsens} accuracy ($-3.9$ \CI{-5.3}{-2.8} vs.\ plain JS), confirming the caution of Sec.~\ref{sec:whywrong}: identifying the irrelevant component is not the same as knowing how to discount it. Finally, corollary~(c) of Sec.~\ref{sec:theory} fails outright: the per-cell mean $\lhat$ correlates \emph{negatively} with Tip-Adapter's test-tuned $\alpha$ (Spearman $-0.22$ to $-0.43$ by backbone), whereas the LOO ratio correlates positively ($+0.30$ to $+0.43$).

\subsection{The Ratio Is Estimable for Free}
\begin{table*}[t]
\caption{Effect of upgrading the text prior from the generic ``a photo of a \{\}.''\ template, paired within (dataset, backbone, $K$, seed) cells ($n{=}1200$ per row; 95\% CI, cluster bootstrap over datasets). Better priors raise zero-shot accuracy, shrink the prototype distance $\smash{\hat d^2}$, and lower the accuracy-optimal blending ratio---and the support-set LOO estimate $\hat\lambda_{\text{LOO}}$ tracks that shift without being told.}
\label{tab:priors}
\centering\footnotesize
\setlength{\tabcolsep}{4pt}
\begin{tabular}{l cccc}
\toprule
Prior upgrade & $\Delta$ zero-shot acc.\ (pp) & $\Delta\, \hat d^2$ & $\Delta\, \lambda^{\text{oracle}}$ & $\Delta\, \hat\lambda_{\text{LOO}}$\\
\midrule
$\to$ dataset-specific & +1.42 [-0.10, +2.94] & -0.0121 [-0.0331, +0.0088] & -0.031 [-0.051, -0.011] & -0.030 [-0.053, -0.008] \\
$\to$ prompt ensemble & +2.11 [+0.46, +3.76] & -0.0156 [-0.0345, +0.0032] & -0.047 [-0.067, -0.026] & -0.041 [-0.059, -0.022] \\
$\to$ LLM descriptions (CuPL) & +2.28 [+0.05, +4.51] & -0.0177 [-0.0233, -0.0122] & -0.026 [-0.058, +0.005] & -0.028 [-0.050, -0.006] \\
\bottomrule
\end{tabular}
\end{table*}

The LOO ratio of Sec.~\ref{sec:estimation} lands within $-0.82$ \CI{-1.31}{-0.32} of the test-set-oracle ratio overall ($-0.63$ \CI{-0.99}{-0.27} at $K{\ge}4$), tracking it with $\rho=0.79$ and MAE $0.12$ (Fig.~\ref{fig:lambda}). It degrades gracefully: on Pets, where trusting the shots is mostly harmful, its paired advantage over zero-shot stays positive at every $K$. It also responds correctly to interventions on the prior (Table~\ref{tab:priors}): upgrading from the generic prompt template raises zero-shot accuracy by $1.4$--$2.3$\,pp and lowers the accuracy-optimal ratio by $0.026$--$0.047$ (the dataset-specific and ensemble tiers resolvably, CuPL grazing zero), while $\hat d^2$ falls on all three tiers in point estimate but resolvably on only one, and the LOO estimate shifts by an almost identical amount without being told. \emph{A single scalar chosen by leave-one-out on the support set reproduces, to within one point, everything that tuning the ratio on the test set can achieve.} Published numbers relying on test-tuned or validation-tuned ratios~\cite{zhang2022tip,wang2024gda} could have been obtained legitimately.

\subsection{The Family Is Capacity-Limited}\label{sec:capacity}
\begin{table*}[t]
\caption{Per-dataset test accuracy (\%) with CLIP ViT-B/16 and the dataset-specific prompt, averaged over five support-set seeds. All methods left of the rule are validation-free; ``Oracle $\lambda$'' grid-searches the global blending ratio \emph{on the test set} and is an unreachable upper bound for the blending family, not a method. Bold marks the best validation-free entry per row. Flowers102 supports only $K{\le}4$ (its pool holds ten images per class) and is therefore absent here; the mean covers the nine datasets with $K{=}16$ cells.}
\label{tab:perdataset}
\centering\small
\setlength{\tabcolsep}{4.2pt}
\begin{tabular}{l ccccccccc !{\vrule} c}
\toprule
 & Zero-shot & NCM & JS blend & LOO blend & GDA$_{\alpha=1}$ & GDA$_{\alpha\text{-LOO}}$ & HOSO & CLAP & LP++ & Oracle $\lambda$\\
\midrule
\multicolumn{11}{l}{\itshape $K=16$ shots}\\
Aircraft & 24.8 & 40.7 & 40.8 & 41.3 & 46.8 & \textbf{49.0} & 42.0 & 45.7 & 45.4 & 41.7 \\
Caltech101 & 82.4 & 88.0 & 88.1 & 88.4 & 91.5 & 91.5 & 91.5 & 91.5 & \textbf{93.1} & 88.8 \\
Cars & 63.7 & 72.3 & 72.5 & 76.4 & \textbf{81.8} & 81.6 & 78.6 & 80.3 & 81.1 & 76.5 \\
DTD & 43.0 & 62.4 & 62.6 & 64.3 & \textbf{72.0} & 70.8 & 70.4 & 71.0 & 71.6 & 64.9 \\
EuroSAT & 55.8 & 76.2 & 76.2 & 76.9 & 86.7 & \textbf{86.7} & 84.6 & 84.0 & 86.5 & 77.5 \\
Food101 & 88.7 & 84.1 & 84.3 & 89.2 & 89.5 & 89.5 & 89.8 & \textbf{90.2} & 89.9 & 89.3 \\
ImageNet & 68.1 & 65.2 & 65.5 & 72.1 & 73.9 & 73.5 & 73.1 & \textbf{74.8} & 74.8 & 72.2 \\
Pets & 89.0 & 76.8 & 77.3 & 91.5 & \textbf{92.8} & 91.7 & 92.4 & 92.6 & 92.5 & 91.7 \\
SUN397 & 60.5 & 68.2 & 68.4 & 71.5 & 75.1 & 74.8 & 73.9 & 74.7 & \textbf{75.3} & 71.6 \\
\midrule
\textit{Mean} & 64.0 & 70.4 & 70.6 & 74.6 & 78.9 & 78.8 & 77.4 & 78.3 & 78.9 & 74.9 \\

\bottomrule
\end{tabular}
\end{table*}

\begin{table}[t]
\caption{Mean accuracy (\%) of validation-free methods by shot count, aggregated over the 45 dataset$\times$backbone combinations present at every $K$ (the combination guard of Sec.~\ref{sec:stats}; Flowers102 is excluded here). \textbf{The Mean column averages the $K\ge2$ column means}, since several estimators need two shots to exist. Tip-Adapter's grid is skipped on the 50 costliest $K{=}16$ cells; its $K{=}16$ entry covers the remaining 175. The lower block lists reference points that consume held-out or test labels; they are bounds, not methods.}
\label{tab:leaderboard}
\centering\small
\setlength{\tabcolsep}{3.6pt}
\resizebox{\columnwidth}{!}{%
\begin{tabular}{l ccccc c}
\toprule
Method & $K{=}1$ & $K{=}2$ & $K{=}4$ & $K{=}8$ & $K{=}16$ & Mean\\
\midrule
Zero-shot & 64.8 & 64.8 & 64.8 & 64.8 & 64.8 & 64.8 \\
NCM (image prototypes) & 40.9 & 52.5 & 61.5 & 67.6 & 71.3 & 63.2 \\
James--Stein blend, Eq.~\eqref{eq:jshat} & -- & 54.1 & 62.4 & 68.1 & 71.4 & 64.0 \\
Centred JS blend & -- & 42.4 & 58.5 & 67.7 & 71.8 & 60.1 \\
LOO blend & -- & 68.4 & 70.9 & 73.4 & 75.1 & 71.9 \\
LOO blend + $K$-corr. & -- & 68.7 & 71.1 & 73.4 & 75.1 & 72.1 \\
GDA ($\alpha{=}1$) & 64.8 & 57.8 & 69.0 & 75.7 & 79.2 & 70.4 \\
GDA ($\alpha$ by LOO) & 64.8 & 64.6 & 70.8 & 75.7 & 79.0 & 72.5 \\
HOSO~\cite{vorster2026hoso} & -- & 68.6 & 72.0 & 75.2 & 77.8 & 73.4 \\
CLAP~\cite{silva2024closer} & 69.0 & 71.5 & 74.2 & 76.5 & 78.5 & 75.2 \\
LP++~\cite{huang2024lppp} & 66.6 & 70.3 & 73.9 & 76.9 & 79.2 & 75.1 \\
\midrule
\multicolumn{7}{l}{\itshape Reference points that use held-out/test labels}\\
GDA ($\alpha$ on held-out val) & 64.8 & 70.2 & 73.8 & 76.8 & 79.5 & 75.1 \\
Oracle blending ratio (test set) & 68.2 & 69.7 & 71.8 & 74.0 & 75.5 & 72.8 \\
Tip-Adapter ($\alpha,\beta$ on test)~\cite{zhang2022tip} & 68.0 & 69.3 & 71.1 & 72.8 & 75.0 & 72.1 \\
MSE-oracle blend, Eq.~\eqref{eq:lamstar} & 44.4 & 54.5 & 62.5 & 68.1 & 71.4 & 64.1 \\
\bottomrule
\end{tabular}
}
\end{table}

\begin{table}[t]
\caption{Paired accuracy difference (percentage points, cluster-bootstrap 95\% CI, resampling the ten datasets) between each validation-free method and the \emph{test-set-oracle} blending ratio---the best any global $\lambda$ can do, obtained by grid search on the test labels. Positive means the method beats a bound the blending family itself cannot reach; bold marks CIs entirely above zero.}
\label{tab:headline}
\centering\small
\setlength{\tabcolsep}{4pt}
\resizebox{\columnwidth}{!}{%
\begin{tabular}{l cc cc}
\toprule
 & \multicolumn{2}{c}{All $K$} & \multicolumn{2}{c}{$K \ge 4$}\\
\cmidrule(lr){2-3}\cmidrule(lr){4-5}
Method vs.\ oracle $\lambda$ & $\Delta$ [95\% CI] & $n$ & $\Delta$ [95\% CI] & $n$\\
\midrule
CLAP & \textbf{+1.92 [+0.87, +2.98]} & 1200 & \textbf{+2.48 [+1.28, +3.69]} & 700 \\
LP++ & \textbf{+1.46 [+0.34, +2.58]} & 1200 & \textbf{+2.78 [+1.23, +4.34]} & 700 \\
HOSO & +0.36 [-0.63, +1.36] & 950 & \textbf{+1.09 [+0.01, +2.17]} & 700 \\
GDA ($\alpha{=}1$) & -2.64 [-4.84, -0.44] & 1200 & +0.86 [-1.68, +3.41] & 700 \\
LOO blend & -0.82 [-1.31, -0.32] & 950 & -0.63 [-0.99, -0.27] & 700 \\
\bottomrule
\end{tabular}
}
\end{table}

Table~\ref{tab:headline} is the headline: averaged over all cells, CLAP exceeds the test-set-oracle blending ratio by $+1.92$ \CI{+0.87}{+2.98} and LP++ by $+1.46$ \CI{+0.34}{+2.58}; at $K\ge4$ all four point estimates are positive and three intervals exclude zero, GDA's excepted ($\CIonly{-1.68}{+3.41}$: above the bound on seven of ten datasets, far below on the two near-ceiling tasks where $\alpha{=}1$ discards a prior it should keep). The intervals are $t$-corrected and coverage-calibrated (Sec.~\ref{sec:stats}); the probes clear the bound at every aggregation we tried. The reference is the \emph{supremum of the family over its contested hyperparameter}, computed with test labels no method may use: a validation-free probe on the same shots beats the best blend that could ever be tuned---and Theorem~\ref{thm:capacity} says why it can.

Against the LOO blend, the honest reference a practitioner could build, the margins widen: CLAP $+3.05$, LP++ $+3.03$, HOSO $+1.18$, GDA ($\alpha$-LOO) $+0.63$\,pp.

\looseness=-1 Per-dataset behaviour (Table~\ref{tab:perdataset}, Fig.~\ref{fig:sweep}) is uniform in direction, instructive in magnitude. On nine of ten datasets at $K{=}16$ a linear probe or GDA leads. The blend's deficit is largest where its inductive bias is weakest: on fine-grained tasks (Aircraft, Cars, DTD) the directions separating confusable classes are not spanned by any interpolation of two fixed prototype sets, and the probes' freedom to rotate per-class boundaries buys five to seven points over the best blend. At the opposite extreme, Food-101 and Pets sit near the zero-shot ceiling and nothing separates families. ImageNet, the one dataset with a thousand classes, follows the aggregate pattern exactly (at $K{=}16$: CLAP and LP++ both $74.8$, against $72.1$ for the LOO blend and $72.2$ for the oracle-tuned blend), so the capacity argument is no artefact of small label spaces. EuroSAT is the clearest case: with ten classes and abundant within-class structure, GDA's covariance-aware head reaches $86.7\%$ at $K{=}16$ where the \emph{oracle-tuned} blend manages $77.5\%$---a nine-point gap no ratio tuning can touch.

\emph{The GDA case study.} Table~\ref{tab:leaderboard} reads the three regimes against one reference: at $K{=}2$ the validation-tuned $\alpha$ reaches $70.2\%$, the leave-one-out $\alpha$ $64.6\%$, and the $\alpha{=}1$ default only $57.8\%$. Fixing the knob costs $12.4$ points; estimating it by the same rule we use for $\lambda$ recovers $6.8$ of them with no forbidden data; the residual cost of refusing validation labels is $5.6$ points at $K{=}2$ and $0.5$ by $K{=}16$. At $K\ge4$ this validation-free GDA also passes the oracle blend.

\subsection{Backbone Consistency}\label{sec:backbones}
\begin{table}[t]
\caption{Backbone consistency, with the mechanism's own predictor beside the damage it was supposed to predict. Columns: mean zero-shot accuracy; the class-independent share of the squared prototype distance (the quantity Eq.~\eqref{eq:mse} mis-charges, over all four prompt tiers); and the paired deficit/surplus of three representative methods against the test-set-oracle blending ratio (pp, cluster-bootstrap 95\% CI over datasets). The three qualitative findings hold on all five backbones. The \emph{ordering} does not: SigLIP has the largest share and the smallest James--Stein deficit, and the deficit tracks zero-shot headroom instead (Sec.~\ref{sec:backbones}).}
\label{tab:perbackbone}
\centering\scriptsize
\setlength{\tabcolsep}{3pt}
\resizebox{\columnwidth}{!}{%
\begin{tabular}{l cc ccc}
\toprule
 & ZS acc. & class-indep. & \multicolumn{3}{c}{vs.\ oracle blending ratio (pp)}\\
\cmidrule(lr){4-6}
Backbone & (\%) & share (\%) & JS blend & LOO blend & CLAP [95\% CI]\\
\midrule
RN50 & 57.3 & 73.9 [68.3, 79.4] & -11.24 & -0.72 & \textbf{+2.64 [+1.38, +3.90]} \\
ViT-B-32 & 60.9 & 77.6 [73.2, 82.0] & -9.65 & -0.78 & \textbf{+2.19 [+1.06, +3.33]} \\
ViT-B-16 & 64.5 & 78.0 [73.8, 82.2] & -8.84 & -0.67 & \textbf{+2.11 [+1.05, +3.17]} \\
ViT-L-14 & 70.6 & 75.7 [72.0, 79.3] & -7.06 & -0.97 & \textbf{+1.40 [+0.41, +2.38]} \\
SigLIP-B-16 & 73.7 & 86.6 [84.6, 88.6] & -5.77 & -0.94 & \textbf{+1.28 [+0.11, +2.46]} \\
\bottomrule
\end{tabular}
}
\end{table}

Table~\ref{tab:perbackbone} repeats the three central comparisons within each backbone. The three qualitative findings hold on all five: the James--Stein blend sits $5.8$--$11.2$\,pp below the oracle ratio, the LOO blend within $1.2$\,pp of it, and CLAP above it with an interval excluding zero.\par \textbf{A second pre-registered prediction fails here, and we report it as such.} Sec.~\ref{sec:whywrong} predicted that the damage would \emph{scale} with the class-independent share of Table~\ref{tab:perbackbone}. It does the opposite: SigLIP has the largest share ($86.6\%$) and the \emph{smallest} deficit ($-5.8$), ResNet-50 the smallest share and the largest deficit ($-11.2$); the rank correlation between share and $|$deficit$|$ is $-0.7$. A simpler variable explains the ordering completely: zero-shot accuracy ($57.3$ to $73.7$ across the five) has rank correlation $-1.00$ with $|$deficit$|$---the damage tracks \emph{headroom}, not gap geometry. The mechanism of Sec.~\ref{sec:whywrong} is therefore supported as a statement about \emph{why the objective is wrong} (Theorem~\ref{thm:shift}, the offset-translation probe of Sec.~\ref{sec:perclass}, and the failure of the centred repair) and refuted as a cross-encoder \emph{dose--response law}. That the diagnosis nevertheless transfers unchanged to a sigmoid-trained model remains the strongest evidence that it attaches to contrastive image--text pre-training as such, not to CLIP.

\subsection{The One-Shot Regime}
At $K{=}1$ half the table collapses: the LOO ratio and HOSO are undefined, $\hat v$ of Eq.~\eqref{eq:jshat} does not exist, and GDA's pooled covariance is zero without augmentation, reducing it to zero-shot exactly. CLAP improves on zero-shot by $4.2$\,pp at $K{=}1$ (Table~\ref{tab:leaderboard}) while NCM \emph{loses} $24.0$\,pp: a single shot is worth more as a constraint on a text-anchored probe than as an image prototype.

\subsection{Computational Cost}\label{sec:cost}
On cached features, fitting costs per ImageNet cell ($C{=}1000$, $N{=}16{,}000$, RTX~4090, float64) are: GDA $0.4$\,s, LP++ $15$\,s, CLAP $39$\,s, HOSO $193$\,s; the LOO ratio costs a single $O(NC)$ sweep (under a second), and the JS ratio is closed-form arithmetic. Feature extraction dominates end-to-end cost; there is no efficiency argument for blending over a probe at this scale.

\begin{figure*}[t]
\centering
\includegraphics[width=0.86\textwidth]{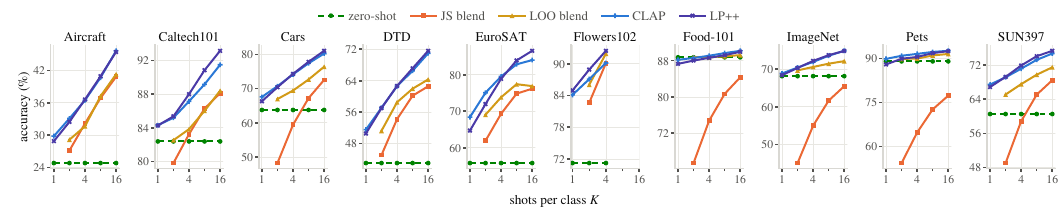}
\caption{Accuracy vs.\ shots for all ten datasets (ViT-B/16, dataset-specific prompt, mean over five seeds): the probes lead almost everywhere; the James--Stein blend hugs or undercuts zero-shot where the prior is strong and forfeits it where weak. Flowers102 ends at $K{=}4$ (Sec.~\ref{sec:datasets}).}
\label{fig:sweep}
\end{figure*}

\subsection{Is One Ratio Too Few?}\label{sec:perclass}
\begin{figure}[t]
\centering
\includegraphics[width=0.79\columnwidth]{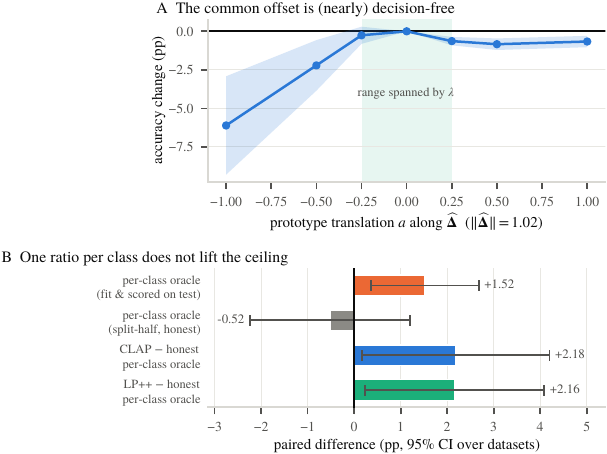}
\caption{Offset and per-class probes. \textbf{A:} accuracy as every prototype is translated by $a\widehat{\bm\Delta}$ before renormalisation, at the oracle ratio; the shaded band is the range over which $\lambda$ itself moves prototypes along that direction. Inside it the common offset is nearly free: what Theorem~\ref{thm:shift}(i) claims for inner-product scoring, and what the MSE objective charges in full. \textbf{B:} the per-class accuracy oracle, fitted and scored on the same test labels and then out of sample.}
\label{fig:probe}
\end{figure}

\looseness=-1 A single global $\lambda$ invites the objection that the family is merely under-parameterised by our choice of oracle. We therefore compute the \emph{per-class} accuracy oracle, one $\lambda_c$ per class, fitted by multi-start coordinate ascent on the test labels (ViT-B/16, ten datasets, five shot counts, three seeds). With $C$ free parameters scored on the labels it was fitted to, it measures memorisation as well as capacity: in sample it gains \PCIN\ over the global oracle; fitted on one test half and scored on the other, only \PCOUT. The gap between those two rows \emph{is} the overfitting, and which side wins is dataset-dependent: out of sample the per-class fit gains on Caltech101 ($+3.0$) and SUN397 ($+2.8$) and loses on ImageNet ($-5.4$), Aircraft ($-2.6$) and DTD ($-1.8$): it helps where classes want different amounts of trust and hurts where $C$ outruns the labels that fit it. Either way the conclusion of Sec.~\ref{sec:capacity} is unmoved: a
validation-free linear probe still beats the honest per-class oracle by \CLAPPC, so the ceiling is
not an artefact of allowing only one ratio.

\looseness=-1 The same probe settles Theorem~\ref{thm:shift} under cosine scoring: translating every prototype by $a\widehat{\bm\Delta}$ at the oracle ratio costs \OFFQ\ at $a{=}{\pm}\tfrac14$ and \OFFH\ at $a{=}{\pm}\tfrac12$ ($\|\widehat{\bm\Delta}\|{=}\DNORM$). But $\bm p_c(\lambda)$ carries $\bm\Delta$ with weight $1-\lambda$, so oracle-to-James--Stein is a translation of $a\!\approx\!-0.63$, where the probe is no longer flat. We therefore ran the counterfactual that decides the question: recompute each cell's oracle-to-James--Stein deficit with the text prototypes as they are, and again with $\widehat{\bm\Delta}$ subtracted from each (estimated on the pool remainder, never the support set; \DCCELLS\ cells, ViT-B/16). The deficit moves from \DCREAL\ to \DCNOD: removing the offset entirely recovers \DCREC, or \DCSHARE\ of the damage. The offset is therefore \emph{not} free at the displacement the ratio actually applies, but it accounts for about a quarter of the loss while the MSE objective charges $78.3\%$ of the squared distance to it, an over-charge of roughly threefold.

\subsection{Two Protocol Objections, Measured}\label{sec:arms}
Both are answered with the released caches rather than argument.

\emph{Augmentation.} Our baselines run on features cached once, while the official CLAP/GDA/LP++/HOSO recipes all consume augmented views. We restored each recipe on a seeded 20-view cache of the support images (\AUGCELLS\ paired cells, identical draws and test sets; \AUGPAIRS\ pairs, all \AUGDS\ datasets). The effect is nil for CLAP (\AUGCLAP) and unresolved for LP++ (\AUGLPP), a point estimate that at face value thins but does not erase its oracle margins (Table~\ref{tab:headline}); HOSO moves \AUGHOSO, inside its precision band. GDA is the exception, in the unexpected direction: augmentation \emph{costs} it \AUGGDAKone\ at $K{=}1$ (a covariance from ten crops of one image is worse than the degenerate fallback) and gains \AUGGDAKtwo\ at $K{=}2$.

\emph{Splits.} Our splits are clip-benchmark's, not the \texttt{split\_zhou} files of the CoOp lineage, so our absolute numbers are not comparable with that literature's tables. Re-extracting on the CoOp splits (\COOPCELLS\ cells; \COOPPAIRS\ pairs of its own ten-dataset grid, all five backbones appearing, ViT-B/16 on every dataset; UCF101 included, ImageNet excluded, as its CoOp protocol needs the full train set) reproduces every qualitative finding: CLAP exceeds the test-set-oracle blending ratio by \COOPCLAP, the LOO blend sits \COOPLOO\ below it, the James--Stein blend \COOPJS, and the class-independent share of the prototype distance is \COOPSHARE.

\looseness=-1 \emph{Modern encoders.} Two 240-cell arms on the full grid ask whether any of this is an artefact of 2021-vintage CLIP. On SigLIP2-B/16~\cite{tschannen2025siglip2}, findings (i) and (ii) reproduce outright ($\lhat$ saturates at \SIGTWOLAMJS\ against an accuracy-optimal \SIGTWOLAMOR), and the capacity finding holds through CLAP, which clears the oracle ratio by \SIGTWOCLAP. LP++ attenuates: \SIGTWOLPP\ at $K{\ge}4$ against the published $+2.78$, and \SIGTWOLPPALL\ over all $K$ (dragged by $K{=}1$); these are intervals a single backbone cannot resolve, so on this encoder the probe half of the headline rests on CLAP. The sharper arm is text-free: DINOv3-B/16~\cite{simeoni2025dinov3} ($86$M parameters, matching ViT-B/16) has no text tower, hence no text prototype and no blending ratio, so it can enter only as a probe-side reference. A plain logistic probe on its features beats the family's test-set oracle by \DINOTHREE\ at $K{\ge}4$, while the identical probe on CLIP features does not (\DINOCLIP): the ceiling of Sec.~\ref{sec:capacity} is a property of the blending family (Theorem~\ref{thm:capacity}), not of CLIP's feature space (per-$K$ tables and protocol in the supplement).

\subsection{Robustness Checks}\label{sec:robust}
\emph{Precision}: float64 re-computation of eight shards leaves every method bit-identical except HOSO (moves up to $3.1$\,pp), so HOSO's $+1.09$ surplus reads as inside its own numerical band; the claims we make rest on CLAP and LP++, which are precision-exact. \emph{Prompt tier}: repeating the $\lambda$-family analyses on the photo, ensemble and CuPL tiers changes levels but no orderings (oracle-vs-LOO gap under a point on every tier; gap share moves two points): the geometry that defeats MSE shrinkage belongs to the encoders, not the prompt. Re-running both probes against each tier's own test-set-oracle ratio (\TIERCELLS\ paired cells, the full \TIERDS-dataset grid) leaves every contrast positive, intervals excluding zero: CLAP \TIERPHOTOCLAP\ on the photo template, \TIERENSCLAP\ on the ensemble and \TIERCUPLCLAP\ on CuPL, LP++ \TIERPHOTOLPP, \TIERENSLPP\ and \TIERCUPLLPP. The margin shrinks monotonically as the prior improves---a better prototype is worth more to the blend than to the probe---but three tiers of it do not close the gap. \emph{Strata}: the headline CLAP-vs-oracle contrast is positive in $43$ of $50$ dataset$\times$backbone strata; the seven exceptions are four Flowers102 strata and three near-ceiling cases, all within $0.3$\,pp of zero---inside median seed noise ($0.4$\,pp).

\section{Discussion}\label{sec:discussion}

\subsection{How to Read the Blending Literature Now}
Blending is not useless. The LOO blend beats zero-shot by $7.6$\,pp and never harms a strong prior, but a one-or-two-point gain over a predecessor, tuned on forbidden labels, lies inside a band free leave-one-out already spans, and wholly below a validation-free probe. New blends should be benchmarked against their family's test-set-oracle ratio and against CLAP or LP++ on identical shots, separating ``a better mechanism'' from ``a better-tuned ratio''---which current evaluation culture~\cite{silva2024closer,huang2024lppp,vorster2026hoso,sheng2025illusion,luo2026fewtrans} does not. This sharpens critiques~\cite{luo2026fewtrans,sheng2025illusion,kravets2025rethinking} that hyperparameter access inflates few-shot gains, adding an existence bound: with unlimited access on its central hyperparameter, the family is still dominated. LOO also removes the excuse that validation-free selection of such scalars is impossible~\cite{luo2026fewtrans}.

\subsection{Scope of the Claim}\label{sec:scope}
The bound covers Eq.~\eqref{eq:blend}, one global interpolation coefficient between the two modality prototype sets, under cosine similarity. It excludes accuracy-tuned per-class coefficients (unbounded test-set overfitting), logit blends with per-branch temperatures and, strictly, sample-level cache models, including Tip-Adapter, whose test-tuned instance we report anyway, and which also falls below the probes (Table~\ref{tab:leaderboard}). The members we evaluate are our own constructions (Sec.~\ref{sec:intro}): published methods in this exact form tune the ratio on forbidden data or fold it into a larger head. The bound is on the \emph{family} and tight by construction: no member beats the test-set-oracle ratio, which is itself beaten. We do not claim linear probes are optimal, only that they are the cheapest widely-available class exceeding the family's supremum. The bound is empirical, with Theorem~\ref{thm:capacity} as its model-level counterpart: in the Gaussian idealisation the family's ceiling is the plane-restricted Mahalanobis separation, and the headroom it implies predicts where the measured surplus is largest. Across ten domains, three orders of magnitude in class count, and six encoders, no counterexample appeared.

\subsection{Why the Bias--Variance Story Misleads, and What Could Rescue It}
The MSE argument fails quantitatively, not conceptually: blending \emph{is} shrinkage, but towards a target whose offset is $78\%$ decision-irrelevant, under a loss that cannot tell. Corollary~\ref{cor:sstar} gives the correct two-class objective and its solution; its multi-class, estimable form is open. A \emph{decision-theoretic shrinkage} would target the pairwise differences $\vt_c-\vt_{c'}$ (Theorem~\ref{thm:shift}); our centred variant, the zeroth-order attempt, failed, showing the correction must be anisotropic, not scalar. An \emph{anisotropic or subspace shrinkage} would estimate which gap directions are decision-relevant, as in the text-aligned projections of~\cite{goswami2026protomix}, but Eq.~\eqref{eq:sstar} needs exactly the quantities ($\bm u$, $\bm\delta$) that $K$ shots estimate worst. Until either exists, Eq.~\eqref{eq:jshat} is a cautionary identity: the most principled-looking closed form here is a positive-part James--Stein rule, and it is precisely wrong. The released records also support oracle-bounding other tuned scalars (TTA temperatures~\cite{shu2022tpt,sheng2025illusion}, interpolation coefficients~\cite{wortsman2022wiseft}, cache sizes~\cite{karmanov2024tda,zhang2024dmn}), several of which we suspect matter as little.

\subsection{Recommendations for Practice}
\emph{(1)} $K\ge2$ shots, no validation data: train CLAP or LP++ on cached features---validation-free, precision-stable, better than every blend we could construct or bound. \emph{(2)} Blend required anyway: set the ratio by Algorithm~\ref{alg:loo}, never on labels the deployment will not have; expect to lose two to three points against a probe. \emph{(3)} $K{=}1$: use CLAP or stay zero-shot; do not blend. \emph{(4)} Evaluating a method that owns a trust parameter: report that parameter's test-set-oracle value alongside it. \emph{(5)} Optimal trust shifts with prior quality (Table~\ref{tab:priors}); a ratio tuned under one prompt tier silently mis-tunes under another, whereas LOO re-adapts by construction.

\subsection{Limitations}\label{sec:limitations}
What this study does not establish. (i)~Our splits are clip-benchmark's, so absolute numbers are not comparable with CoOp-lineage tables; the CoOp-split replication (Sec.~\ref{sec:arms}) shows the \emph{conclusions} transfer, not the levels. (ii)~The augmentation and split arms subsample the grid; their intervals lean on the covered clusters. (iii)~HOSO, re-implemented from its paper (no public code), is the one precision-sensitive method; its comparisons carry a ${\sim}0.5$\,pp band. (iv)~CLAP's anchor scale is fixed rather than recovered. (v)~The per-class accuracy oracle (Sec.~\ref{sec:perclass}) is a multi-start coordinate ascent, not a certified optimum, and its split-half version is ViT-B/16 only. (vi)~Pre-training contamination~\cite{kravets2025rethinking} affects all methods equally but is not corrected for. (vii)~Flowers102 stops at $K{=}4$ and UCF101 is absent. (viii)~Our scope is single-label classification from cached global features. None of these, in our judgement, threatens the three findings, which rest on paired within-cell comparisons that these factors shift jointly rather than differentially.

\section{Conclusion}\label{sec:conclusion}
We set out with a simple question: how much is the blending ratio at the heart of few-shot vision--language adaptation actually worth? Across 4{,}800 paired comparisons, we found the answer to be remarkably consistent. The model family's own bias--variance argument leads to a positive-part James--Stein rule, but that rule optimises prototype error rather than classification, and $78\%$ of the bias it seeks to correct is decision-irrelevant. The ratio that accuracy actually prefers can be estimated from the support set alone. We validated this through leave-one-out, to within one point of a test-set oracle, removing the need to tune on held-out or forbidden labels. Importantly, even the oracle-tuned blend is beaten by validation-free linear probes. The bottleneck is therefore not how one chooses the blending ratio.

The findings suggest a rethink on hyperparameters. The ratio is not an unresolved tuning problem. It can be estimated accurately. Our observation is that the literature has been placing its efforts on precisely this optimisation, while the key question is whether the model family itself is best suited to the task at hand. Future work should reconsider whether simple prototype blending is an adequate adaptation mechanism at all.

\ifCLASSOPTIONcompsoc
  \section*{Acknowledgments}
\else
  \section*{Acknowledgment}
\fi
The authors thank the maintainers of the open-source implementations reproduced in Section~\ref{sec:protocol}.

\balance
\bibliographystyle{IEEEtran}
\bibliography{refs}

\clearpage
\setcounter{section}{0}
\setcounter{equation}{0}
\setcounter{table}{0}
\renewcommand{\thesection}{S\arabic{section}}
\renewcommand{\theequation}{S\arabic{equation}}
\renewcommand{\thetable}{S\arabic{table}}

\twocolumn[
  \begin{@twocolumnfalse}
  \vspace{1em}
  \begin{center}
  {\Large\bfseries Supplementary Material\par}
  \vspace{0.5em}
  \end{center}
  \vspace{1em}
  \end{@twocolumnfalse}
]

\section{Overview and Notation}
\suppself\ contains the complete proofs of Propositions~1--5 and Theorems~1--2 of the main text, the Monte Carlo protocol that verified every formula before it entered the manuscript, the per-stratum table behind the capacity and sampling measurements of Secs.~3.3 and~3.5, and the per-shot tables of the modern-encoder arm of Sec.~6.8. Numbering of the form ``Proposition~3'' or ``Eq.~(11)'' refers to the main text; equations local to \suppselflower\ are numbered (S1), (S2), \dots. Notation follows the main text throughout: for a fixed class, $\bm x_1,\dots,\bm x_K\in\mathbb R^d$ are the support features with mean $\mustar$ and covariance $\bm\Sigma^*$, $\muhat$ is their average, $\vt$ the text prototype, $g^2=\|\vt-\mustar\|^2$, $v=\operatorname{tr}(\bm\Sigma^*)/K$, $\hat d^2=\|\vt-\muhat\|^2$, $\hat v=\operatorname{tr}(\hat{\bm\Sigma})/K$, and $\lhat=(1-\hat v/\hat d^2)_+$. Assumption (A1) is the isotropic Gaussian model $\bm x_i\sim\mathcal N(\mustar,\sigma^2\bm I_d)$ with $\sigma^2=vK/d$. Every result proved here is also checked numerically by \texttt{code/verify\_capacity.py} (Sec.~\ref{ssec:mc}).

\section{Proofs for Section 3.2 of the Main Text}

\restate{Proposition 1 (MSE-optimal ratio, restated).}{$\operatorname{MSE}(\lambda)=(1-\lambda)^2g^2+\lambda^2v$ is strictly convex with unique minimiser $\lambda^*=g^2/(g^2+v)$ and value $\operatorname{MSE}(\lambda^*)=g^2v/(g^2+v)$.}

\begin{IEEEproof}
$\operatorname{MSE}''=2(g^2+v)>0$, so the objective is strictly convex; setting $\operatorname{MSE}'(\lambda)=2\lambda(g^2+v)-2g^2=0$ gives $\lambda^*$, and substitution gives the value:
$(1-\lambda^*)^2g^2+\lambda^{*2}v=\frac{v^2g^2+g^4v}{(g^2+v)^2}=\frac{g^2v}{g^2+v}$.
\end{IEEEproof}

\restate{Proposition 2 (Plug-in, restated).}{$\mathbb E[\hat v]=v$ and $\mathbb E[\hat d^2]=g^2+v$, so $\hat g^2=\hat d^2-\hat v$ is unbiased for $g^2$; substituting $(\hat g^2,\hat v)$ into $\lambda^*$ and clipping to $[0,1]$ gives $\lhat=(1-\hat v/\hat d^2)_+$.}

\begin{IEEEproof}
Unbiasedness of the within-class scatter gives $\mathbb E[\operatorname{tr}\hat{\bm\Sigma}]=\operatorname{tr}\bm\Sigma^*$, hence $\mathbb E[\hat v]=v$. Writing $\vt-\muhat=(\vt-\mustar)-(\muhat-\mustar)$,
\begin{equation}
\mathbb E\|\vt-\muhat\|^2=g^2-2\,\mathbb E\langle\vt-\mustar,\muhat-\mustar\rangle+\mathbb E\|\muhat-\mustar\|^2=g^2+v,
\end{equation}
because the middle term vanishes ($\muhat$ is unbiased and $\vt-\mustar$ is deterministic) and $\mathbb E\|\muhat-\mustar\|^2=\operatorname{tr}\operatorname{Cov}[\muhat]=\operatorname{tr}(\bm\Sigma^*)/K$. Substitution: $\hat g^2/(\hat g^2+\hat v)=(\hat d^2-\hat v)/\hat d^2=1-\hat v/\hat d^2$, the internal $\hat v$ cancelling; clipping to $[0,1]$ yields the positive part.
\end{IEEEproof}

\restate{Proposition 3 (Dominance at the family's own objective, restated).}{Under (A1), if $K\ge2$ and $(K-1)(d-4)\ge2$, then $\mathbb E\|\bm p(\lhat)-\mustar\|^2<\mathbb E\|\muhat-\mustar\|^2$ for every $(\mustar,\vt)$.}

\begin{IEEEproof}
Work in the coordinates of the classical problem. Set
\begin{equation}
\bm X=\muhat-\vt,\qquad
\bm\theta=\mustar-\vt,\qquad
\tau^2=\sigma^2/K=v/d,
\end{equation}
so that $\bm X\sim\mathcal N(\bm\theta,\tau^2\bm I_d)$ and $\hat d^2=\|\bm X\|^2$. Let $S=\frac1K\sum_i\|\bm x_i-\muhat\|^2$; then $S/\tau^2\sim\chi^2_m$ with $m=d(K-1)$, and $S$ is independent of $\bm X$ (independence of sample mean and scatter for Gaussian samples). Since $\hat v=\operatorname{tr}\hat{\bm\Sigma}/K=\frac{1}{K(K-1)}\sum_i\|\bm x_i-\muhat\|^2=\frac{d}{m}S$, the blend error is
\begin{equation}
\bm p(\lhat)-\mustar=\lhat\bm X-\bm\theta,\qquad
\lhat=\Bigl(1-\tfrac{c\,S}{\|\bm X\|^2}\Bigr)_{+},\qquad c=\tfrac dm .
\end{equation}
This is Baranchik's estimator of $\bm\theta$ with $r\equiv1$~\cite{baranchik1970family}; we give the risk computation in full. Consider first the unclipped rule $\bm\delta_c=(1-cS/\|\bm X\|^2)\,\bm X$. Its risk differs from that of $\bm X$ (which is $d\tau^2$) by
\begin{equation}
\mathbb E\|\bm\delta_c-\bm\theta\|^2-d\tau^2
=\mathbb E\Bigl[\frac{c^2S^2}{\|\bm X\|^2}\Bigr]-2c\,\mathbb E\Bigl[S\,\frac{(\bm X-\bm\theta)^{\!\top}\bm X}{\|\bm X\|^2}\Bigr].
\label{eq:riskdiff}
\end{equation}
For the second term, apply Stein's identity $\mathbb E[(\bm X-\bm\theta)^{\!\top}\bm h(\bm X)]=\tau^2\,\mathbb E[\nabla\!\cdot\!\bm h(\bm X)]$ to $\bm h(\bm x)=\bm x/\|\bm x\|^2$, whose divergence is $(d-2)/\|\bm x\|^2$ (integrability requires $d\ge3$, which the hypothesis implies: $(K-1)(d-4)\ge2$ forces $d\ge6$ at $K{=}2$ and $d\ge5$ for every $K\ge3$). With $S\perp\bm X$, $\mathbb E[S]=m\tau^2$ and $\mathbb E[S^2]=m(m+2)\tau^4$, Eq.~\eqref{eq:riskdiff} becomes
\begin{equation}
m\tau^4\,\mathbb E\Bigl[\tfrac{1}{\|\bm X\|^2}\Bigr]\,\bigl[c^2(m+2)-2c(d-2)\bigr].
\label{eq:riskfactor}
\end{equation}
The bracket is non-positive iff $c\le2(d-2)/(m+2)$. With $c=d/m$ and $m=d(K-1)$,
\begin{equation}
\tfrac dm\le\tfrac{2(d-2)}{m+2}
\iff md\ge2d+4m
\iff (K-1)(d-4)\ge2 ,
\end{equation}
dividing by $d>0$ in the last step. When the inequality is strict, \eqref{eq:riskfactor} is strictly negative for every $\bm\theta$, because $\mathbb E[1/\|\bm X\|^2]>0$; when it holds with equality, the unclipped rule merely matches $\bm X$. In either case the positive part decides: $\lhat$ differs from the unclipped coefficient on the event $\{cS>\|\bm X\|^2\}$, which has positive probability for every $\bm\theta$, and replacing a negative coefficient by zero strictly reduces the conditional squared error on that event (the standard positive-part argument, e.g.~\cite{lehmann1998theory}, Thm.~5.7.4 region). Hence
$\mathbb E\|\lhat\bm X-\bm\theta\|^2<\mathbb E\|\bm X-\bm\theta\|^2$, which is the claim after translating back by $\vt$.
\end{IEEEproof}

\section{Proofs for Section 3.3 of the Main Text}

\restate{Proposition 4 (Sampling distribution of $\lhat$, restated).}{Under (A1), $\hat d^2$ and $\hat v$ are independent, $\hat d^{2}\sim\frac{v}{d}\chi^{2}_{d}(g^{2}d/v)$, $\hat v\sim v\,\chi^{2}_{d(K-1)}/(d(K-1))$, and the delta method gives Eq.~(10) of the main text.}

\begin{IEEEproof}
$\muhat\sim\mathcal N(\mustar,\tfrac vd\bm I_d)$, so $\hat d^2=\|\vt-\muhat\|^2$ is a scaled non-central $\chi^2$ with $d$ degrees of freedom and non-centrality $g^2d/v$; the scatter is an independent (Cochran) scaled central $\chi^2$ with $d(K-1)$ degrees of freedom. From $\operatorname{Var}[\chi^2_m(\delta)]=2m+4\delta$,
\begin{equation}
\operatorname{Var}[\hat d^{2}]=\tfrac{2v}{d}(v+2g^{2}),\qquad
\operatorname{Var}[\hat v]=\tfrac{2v^{2}}{d(K-1)} .
\label{eq:a1vars}
\end{equation}
For Eq.~(10), expand $\lhat=1-\hat v/\hat d^2$ around $(\mathbb E\hat v,\mathbb E\hat d^2)=(v,g^2+v)$: the gradient is $\bigl(-\tfrac1{g^2+v},\,\tfrac{v}{(g^2+v)^2}\bigr)$, the mean maps to $1-\tfrac{v}{g^2+v}=\lambda^*$, and independence gives
\begin{equation}
\operatorname{Var}[\lhat]\approx\frac{\operatorname{Var}[\hat v]}{(g^2+v)^2}+\frac{v^2\operatorname{Var}[\hat d^2]}{(g^2+v)^4}.
\label{eq:delta}
\end{equation}
With $1-\lambda^*=\tfrac{v}{g^2+v}$ the first term is $(1-\lambda^*)^2\tfrac{2}{d(K-1)}$ and, using $v(v+2g^2)=(g^2+v)^2-g^4=(g^2+v)^2(1-\lambda^{*2})$, the second is $(1-\lambda^*)^2\tfrac{2}{d}(1-\lambda^{*2})$. Summing and taking the square root gives Eq.~(10); the $O(d^{-1})$ remainder is the second-order delta-method term.
\end{IEEEproof}

\restate{Proposition 5 (Sampling law under arbitrary covariance, restated).}{For $\bm x_i\sim\mathcal N(\mustar,\bm\Sigma^*)$ with any $\bm\Sigma^*\succeq0$, $\hat d^2$ and $\hat v$ remain independent, with $\bm b=\vt-\mustar$: $\operatorname{Var}[\hat d^{2}]=\tfrac{4}{K}\bm b^{\!\top}\bm\Sigma^*\bm b+\tfrac{2}{K^{2}}\operatorname{tr}(\bm\Sigma^{*2})$, $\operatorname{Var}[\hat v]=2\operatorname{tr}(\bm\Sigma^{*2})/(K^{2}(K-1))$, and the delta method gives Eq.~(11) of the main text with $\deff=(\operatorname{tr}\bm\Sigma^*)^{2}/\operatorname{tr}(\bm\Sigma^{*2})$, reducing exactly to Eq.~(10) when $\bm\Sigma^*=\sigma^2\bm I$.}

\begin{IEEEproof}
Independence of the sample mean and the sample covariance holds for Gaussian samples under any $\bm\Sigma^*$, hence $\hat d^2\perp\hat v$.

\emph{Variance of $\hat d^2$.} Write $\muhat=\mustar+\bm\varepsilon$ with $\bm\varepsilon\sim\mathcal N(\bm0,\bm\Sigma^*/K)$. Then
$\hat d^2=\|\bm b-\bm\varepsilon\|^2=g^2-2\bm b^{\!\top}\bm\varepsilon+\|\bm\varepsilon\|^2$.
The linear and quadratic parts are uncorrelated (third moments of a centred Gaussian vanish), so
\begin{equation}
\operatorname{Var}[\hat d^2]
=4\bm b^{\!\top}\tfrac{\bm\Sigma^*}{K}\bm b+2\operatorname{tr}\!\bigl(\tfrac{\bm\Sigma^{*}}{K}\bigr)^{\!2}
=\tfrac4K\bm b^{\!\top}\bm\Sigma^*\bm b+\tfrac{2}{K^2}\operatorname{tr}(\bm\Sigma^{*2}),
\end{equation}
using $\operatorname{Var}[\|\bm\varepsilon\|^2]=2\operatorname{tr}(\operatorname{Cov}[\bm\varepsilon]^2)$ for Gaussian $\bm\varepsilon$.

\emph{Variance of $\hat v$.} $(K-1)\hat{\bm\Sigma}=\sum_{j=1}^{K-1}\bm z_j\bm z_j^{\!\top}$ in distribution, with $\bm z_j\sim\mathcal N(\bm0,\bm\Sigma^*)$ i.i.d.\ (Wishart representation), so $\operatorname{tr}[(K-1)\hat{\bm\Sigma}]=\sum_j\|\bm z_j\|^2$ is a sum of $K-1$ i.i.d.\ terms each with variance $2\operatorname{tr}(\bm\Sigma^{*2})$; dividing by $K(K-1)$ gives
$\operatorname{Var}[\hat v]=2\operatorname{tr}(\bm\Sigma^{*2})/(K^2(K-1))$.

\emph{Delta method.} Identical to \eqref{eq:delta} with the new variances. For the first term, $v=\operatorname{tr}\bm\Sigma^*/K$ gives
\begin{equation}
\frac{\operatorname{Var}[\hat v]}{(g^2+v)^2}
=\frac{(1-\lambda^*)^2\,2\operatorname{tr}(\bm\Sigma^{*2})}{(K-1)(\operatorname{tr}\bm\Sigma^*)^2}
=\frac{(1-\lambda^*)^2\,2}{(K-1)\deff},
\end{equation}
and the second term is $(1-\lambda^*)^2\operatorname{Var}[\hat d^2]/(g^2+v)^2$ exactly as before; together these are Eq.~(11).

\emph{Reduction under isotropy.} With $\bm\Sigma^*=\sigma^2\bm I$ and $\sigma^2=vK/d$: $\deff=d$, $\bm b^{\!\top}\bm\Sigma^*\bm b=\sigma^2g^2$, $\operatorname{tr}(\bm\Sigma^{*2})=d\sigma^4$, so the second radicand term is
$\bigl(\tfrac{4vg^2}{d}+\tfrac{2v^2}{d}\bigr)/(g^2+v)^2=\tfrac2d\,\tfrac{v(2g^2+v)}{(g^2+v)^2}=\tfrac2d(1-\lambda^{*2})$,
recovering Eq.~(10) term by term.
\end{IEEEproof}

\section{Proof for Section 3.4 of the Main Text}

\restate{Theorem 1 (Common shifts are (almost) decision-free, restated).}{Shift every prototype by the same $\bm\Delta$. (i)~Inner-product scoring: predictions are identical. (ii)~Nearest-prototype scoring: decision regions translate rigidly, and the induced risk depends on $\bm\Delta$ only through $\Pi_{\mathcal W}\bm\Delta$, $\mathcal W=\operatorname{span}\{\bm w_{cc'}\}$. (iii)~If $\bm\Delta$ is isotropic relative to $\mathcal W$, $\mathbb E\langle\bm\Delta,\bm w_{cc'}\rangle^2=\|\bm\Delta\|^2\|\bm w_{cc'}\|^2/d$.}

\begin{IEEEproof}
(i) $\langle\bm x,\bm p_c+\bm\Delta\rangle-\langle\bm x,\bm p_{c'}+\bm\Delta\rangle=\langle\bm x,\bm w_{cc'}\rangle$: every pairwise score difference is unchanged for every input, so the arg\,max is unchanged.
(ii) $\|\bm x-(\bm p_c+\bm\Delta)\|^2=\|(\bm x-\bm\Delta)-\bm p_c\|^2$, so the shifted classifier is the original applied to $\bm x-\bm\Delta$: decision regions translate rigidly by $\bm\Delta$. Membership in a region is decided by the affine tests $\langle\bm x,\bm w_{cc'}\rangle\ge\theta_{cc'}$, whose normals span $\mathcal W$; since $\langle\bm\Delta,\bm w_{cc'}\rangle=\langle\Pi_{\mathcal W}\bm\Delta,\bm w_{cc'}\rangle$ for every pair, the translated regions---and hence the induced risk under any input distribution---depend on $\bm\Delta$ only through $\Pi_{\mathcal W}\bm\Delta$.
(iii) Isotropy gives $\mathbb E[\bm\Delta\bm\Delta^{\!\top}]=\tfrac{\|\bm\Delta\|^2}{d}\bm I_d$, so $\mathbb E\langle\bm\Delta,\bm w\rangle^2=\bm w^{\!\top}\mathbb E[\bm\Delta\bm\Delta^{\!\top}]\bm w=\|\bm\Delta\|^2\|\bm w\|^2/d$.
\end{IEEEproof}

\medskip\noindent\textbf{The noise-aware deflection objective (Corollary 2 of the main text).} With estimated image prototypes, $\muhat_\pm=\mustar_\pm+\bm\varepsilon_\pm$, $\bm\varepsilon_\pm\sim\mathcal N(\bm0,\tfrac vd\bm I_d)$, the score variance acquires the estimation noise, and to $O(v/d)$ the denominator $B$ of the deflection ratio becomes
\begin{equation}
B_v(s)=\sigma^{2}\bigl[(\|\bm u\|^{2}{+}2v)+4(\gamma{-}v)\,s+(4\|\bm\delta\|^{2}{+}2v)\,s^{2}\bigr].
\end{equation}
For $\beta=\eta=0$ the numerator is $A(s)=\tfrac12\|\bm u\|^2+\gamma s$, so the stationarity condition $2A'B_v=AB_v'$ is linear-quadratic rather than cubic; collecting terms, the $s^2$ coefficients cancel and the unique root is Eq.~(16) of the main text, clipped to $[0,1]$. Setting $v=0$ recovers $s^*=0$ whenever $\|\bm u\|^2\|\bm\delta\|^2>\gamma^2$ (Cauchy--Schwarz, strict unless $\bm\delta\propto\bm u$), the noiseless statement of Corollary~1.

\section{Proof for Section 3.5 of the Main Text}

\restate{Theorem 2 (Capacity of the blending family, restated).}{Two classes, $\bm x\,|\,c\sim\mathcal N(\mustar_c,\bm\Sigma)$, $\bm\Sigma\succ0$ shared, balanced priors. A linear rule with direction $\bm w$ (oriented so $\bm u^{\!\top}\bm w>0$) and its optimal threshold errs with probability $\Phi(-m(\bm w)/2)$, $m(\bm w)=\bm u^{\!\top}\bm w/\sqrt{\bm w^{\!\top}\bm\Sigma\bm w}$. With $a_1=\bm u^{\!\top}\bm u$, $a_2=\bm u^{\!\top}\bm\delta$, and $q_{11},q_{12},q_{22}$ the $\bm\Sigma$-quadratic forms in $(\bm u,\bm\delta)$: (i)~on the path $\bm w(s)=\bm u+2s\bm\delta$, the unique stationary point of $m^2$ where $m>0$ is $s^{\star}=\frac{a_1q_{12}-a_2q_{11}}{2(a_2q_{12}-a_1q_{22})}$, and the maximum over $s\in[0,1]$ is attained at $s^\star$ or an endpoint; (ii)~$\max_{s}m(\bm w(s))^{2}\le\max_{\bm w\in\operatorname{span}\{\bm u,\bm\delta\}}m(\bm w)^{2}=\bm c^{\!\top}\bm Q^{-1}\bm c\le\bm u^{\!\top}\bm\Sigma^{-1}\bm u$, with equality on the right iff $\bm\Sigma^{-1}\bm u\in\operatorname{span}\{\bm u,\bm\delta\}$; (iii)~if $\bm\Sigma=\sigma^2\bm I$ then $s^{\star}=0$.}

\begin{IEEEproof}
\emph{Error formula.} For the rule $\operatorname{sign}(\bm w^{\!\top}\bm x-\theta)$ the class-conditional errors are $\Phi\bigl((\theta-\bm w^{\!\top}\mustar_+)/\sqrt{\bm w^{\!\top}\bm\Sigma\bm w}\bigr)$ and $\Phi\bigl((\bm w^{\!\top}\mustar_--\theta)/\sqrt{\bm w^{\!\top}\bm\Sigma\bm w}\bigr)$. Their balanced average is minimised at the midpoint $\theta^\star=\bm w^{\!\top}(\mustar_++\mustar_-)/2$ (the two arguments are then equal and $\Phi$ is convex on the relevant branch), giving $\operatorname{err}=\Phi(-m(\bm w)/2)$ with $m$ as stated. Since $\Phi$ is decreasing in $m$, maximising $m^2$ subject to the orientation $\bm u^{\!\top}\bm w>0$ minimises the error.

\emph{(i).} On the path, write $N(s)=(a_1+2sa_2)^2$ and $D(s)=q_{11}+4sq_{12}+4s^2q_{22}>0$, so $m^2=N/D$ and
\begin{equation}
\Bigl(\frac ND\Bigr)^{\!\prime}=\frac{4\,(a_1+2sa_2)\bigl[a_2D-(a_1+2sa_2)(q_{12}+2sq_{22})\bigr]}{D(s)^2}.
\end{equation}
The factor $a_1+2sa_2$ vanishes only where $m=0$, the global minimum of $m^2$; on the region $m>0$, stationarity requires the bracket to vanish. Expanding the bracket, the $s^2$ terms cancel and what remains is linear:
\begin{equation}
a_2q_{11}+2s\,a_2q_{12}=a_1q_{12}+2s\,a_1q_{22},
\end{equation}
whose unique root is $s^\star$ as displayed, provided $a_2q_{12}\ne a_1q_{22}$; if $a_2q_{12}=a_1q_{22}$ the bracket never vanishes, $m^2$ is monotone wherever $m>0$, and the maximum over $[0,1]$ sits at an endpoint. In all cases a continuous function on a compact interval attains its maximum at an interior stationary point or an endpoint.

\emph{(ii).} The first inequality holds because $\bm w(s)\in\operatorname{span}\{\bm u,\bm\delta\}$ for every $s$. For the subspace maximum, let $\bm V=[\bm u\;\bm\delta]$, $\bm c=\bm V^{\!\top}\bm u=(a_1,a_2)^{\!\top}$, $\bm Q=\bm V^{\!\top}\bm\Sigma\bm V$, and substitute $\bm w=\bm V\bm z$:
\begin{equation}
\max_{\bm z\ne\bm0}\frac{(\bm c^{\!\top}\bm z)^2}{\bm z^{\!\top}\bm Q\bm z}
=\max_{\bm y\ne\bm0}\frac{\bigl((\bm Q^{-1/2}\bm c)^{\!\top}\bm y\bigr)^2}{\|\bm y\|^2}
=\bm c^{\!\top}\bm Q^{-1}\bm c,
\label{eq:rayleigh}
\end{equation}
by Cauchy--Schwarz after $\bm y=\bm Q^{1/2}\bm z$, with equality iff $\bm y\propto\bm Q^{-1/2}\bm c$, i.e.\ $\bm z\propto\bm Q^{-1}\bm c$. (When $\bm\delta\propto\bm u$, $\bm Q$ is singular but the span is the line through $\bm u$ and the same computation applies in one dimension.) The identical computation over all of $\mathbb R^d$ gives $\max_{\bm w}m(\bm w)^2=\bm u^{\!\top}\bm\Sigma^{-1}\bm u$, attained exactly on the ray $\bm w\propto\bm\Sigma^{-1}\bm u$. A maximum over a subspace never exceeds the global one; and since the global maximiser is unique up to scale, equality holds iff the subspace contains it, i.e.\ iff $\bm\Sigma^{-1}\bm u\in\operatorname{span}\{\bm u,\bm\delta\}$.

\emph{(iii).} If $\bm\Sigma=\sigma^2\bm I$ then $q_{11}=\sigma^2a_1$, $q_{12}=\sigma^2a_2$, $q_{22}=\sigma^2\|\bm\delta\|^2$, so the numerator of $s^\star$ is $\sigma^2(a_1a_2-a_2a_1)=0$ while the denominator is $2\sigma^2(a_2^2-a_1\|\bm\delta\|^2)<0$ strictly unless $\bm\delta\propto\bm u$ (Cauchy--Schwarz); hence $s^\star=0$. It is the maximum on $[0,1]$: $m^2(0)=a_1/\sigma^2$, the only stationary point with $m>0$ is $s^\star=0$, and $m^2(s)\to a_2^2/(\sigma^2\|\bm\delta\|^2)<m^2(0)$ as $s\to\infty$, so $m^2$ decreases away from $0$. In the degenerate case $\bm\delta=\kappa\bm u$ the path margin is constant in $s$ and $s^\star=0$ attains it trivially (the released code adopts the same convention when the denominator vanishes numerically).
\end{IEEEproof}

\section{Validation and Verification Protocols}\label{ssec:mc}

\emph{Synthetic validation of the estimators (Sec.~3.3 of the main text).} Every statistical component was validated on synthetic ground truth before touching real features: $80$ configurations crossing $d$, $K$, $g$ and isotropic versus power-law spectra, $3{,}000$ replicates each; $\hat g^2$ and $\hat v$ pass unbiasedness $z$-tests within budget, and the analytic $\lambda^*$ matches numerical minimisation to $10^{-6}$ (\texttt{code/synth\_validate.py}, \texttt{code/verify\_theory\_ext.py}). On real features, the per-class $\lhat$ tracks the pool-estimated $\lambda^*$ with Spearman $\rho\approx0.59$ and MAE $0.007$.

\emph{Monte Carlo verification of the capacity, dominance and sampling results.} Every formula proved above was verified numerically before entering the manuscript; the checks live in \texttt{code/verify\_capacity.py} (deterministic seed, single command) and all pass. In brief:

\emph{V1, error formula and threshold inequality.} Two anisotropic Gaussian classes ($d{=}40$, power-law spectrum, a deliberately non-zero common offset $\bm\Delta$), $4\times10^5$ samples: the measured error of the optimal-threshold rule matches $\Phi(-m/2)$ to $<3\times10^{-3}$ at five points along the blend path, and the blend's own prototype-midpoint threshold never beats the optimal threshold, so the theorem's $\Phi(-m/2)$ is a valid lower bound for the blend's achievable error.

\emph{V2, closed form $s^\star$.} 300 random geometries ($d\in[6,60)$, spectral decay in $[0,2]$): the closed form attains the maximum of $m^2$ found by a $2\times10^4$-point grid to within $10^{-6}$ in margin$^2$; under $\bm\Sigma=\bm I$, $|s^\star|<10^{-10}$ in all trials.

\emph{V3, nested chain.} The same 300 geometries: no violation of $m^2_{\text{path}}\le m^2_{\text{plane}}\le m^2_{\text{full}}$ beyond $10^{-8}$; the plane-to-full gap is generically strict; and a constructed case with $\bm\Sigma^{-1}\bm u\in\operatorname{span}\{\bm u,\bm\delta\}$ attains equality to relative error $<10^{-8}$.

\emph{V4, dominance.} Eight $(d,K,g)$ configurations spanning $d\in\{8,\dots,512\}$, $K\in\{2,3,4,16\}$, $g\in[0,20]$, $2\times10^4$ replicates each: the measured MSE gain of $\bm p(\lhat)$ over $\muhat$ is positive (within three standard errors) at every configuration satisfying $(K-1)(d-4)\ge2$.

\emph{V5, generalised SE.} Six covariance spectra at $d{=}256$ (decay $0$ to $2$, plus a spiked case with $\bm b$ aligned to the top eigenvector---the adversarial case for a first-order expansion), $4\times10^4$ replicates: the predicted SE matches the Monte Carlo SD within $8\%$ everywhere, the worst case being the spiked configuration, where the discrepancy is the second-order delta-method remainder (halving the spike halves it), not a formula error; and the general formula reproduces the (A1) formula under $\bm\Sigma=\sigma^2\bm I$ to relative error $<10^{-10}$.

\section{Per-Stratum Capacity and Sampling Measurements}
Table~\ref{tab:strata} reports, for each of the $50$ (dataset, backbone) strata, the quantities behind Secs.~3.3 and~3.5 of the main text; \texttt{code/probe\_capacity.py} regenerates all of them from the released features and records. Per stratum: the pooled within-class covariance $\bm\Sigma$ is estimated from mean-centred pool residuals (up to $100$ per class) with oracle-approximating shrinkage~\cite{chen2010shrinkage}; the three nested separations of Theorem~2 are evaluated in closed form for every class pair, converted to error by $\Phi(-m/2)$, and averaged; \emph{ratio} is the mean $\sqrt{m^2_{\text{path}}/m^2_{\text{full}}}$ and \emph{headroom} the mean error gap between the best path direction and the unrestricted Mahalanobis direction, in percentage points. \emph{Surplus} is the measured CLAP-minus-oracle-blend accuracy gap from the main matrix (dataset-specific prompt tier, all $K$ and seeds). For the sampling law, per-class $\operatorname{tr}\bm\Sigma$, $\operatorname{tr}\bm\Sigma^2$ (Chen--Qin unbiased estimator~\cite{chen2010twosample}) and $\bm b^{\!\top}\bm\Sigma\bm b$ are estimated from pool residuals via Gram matrices; the predicted seed-to-seed SD of $\lhat$ (Eq.~(11)) is the median over classes and is compared with the SD measured across the five recorded support draws (median over classes; Flowers102 has no $K{=}16$ arm). Across strata the measured-to-predicted closure is \SECLOSEKtwo\ at $K{=}2$ and \SECLOSEKsixteen\ at $K{=}16$; $\deff$ has median \DEFFMED\ against nominal $d\in\{512,768,1024\}$; and predicted headroom meets measured surplus at Spearman \CAPSPEAR\ \CAPSPEARCI\ (cluster bootstrap over datasets, Student-$t$ critical value), \CAPPARTIAL\ after rank-partialling out stratum difficulty.

\begin{table*}[t]
\centering
\caption{Per-stratum capacity and sampling measurements ($50$ strata; generated by \texttt{code/probe\_capacity.py} and \texttt{code/capacity\_macros.py}). Ratio: mean path-to-full margin ratio. Headroom: mean error gap (pp) between the best blend-path direction and the unrestricted Mahalanobis direction. Surplus: measured CLAP $-$ oracle-blend accuracy (pp). $\deff$: median per-class participation ratio. SD columns: measured vs.\ predicted (Eq.~(11)) seed-to-seed SD of $\lhat$.}
\label{tab:strata}
\footnotesize
\setlength{\tabcolsep}{4.5pt}
\begin{tabular}{llcccc cc cc}
\toprule
& & & Headroom & Surplus & & \multicolumn{2}{c}{SD of $\lhat$, $K{=}2$} & \multicolumn{2}{c}{SD of $\lhat$, $K{=}16$}\\
\cmidrule(lr){7-8}\cmidrule(lr){9-10}
Dataset & Backbone & Ratio & (pp) & (pp) & $\deff$ & measured & predicted & measured & predicted \\
\midrule
Aircraft & RN50 & 0.47 & 1.79 & +2.69 & 32 & $2.1\times10^{-2}$ & $1.9\times10^{-2}$ & $6.1\times10^{-4}$ & $6.8\times10^{-4}$ \\
Aircraft & ViT-B/32 & 0.54 & 1.15 & +2.76 & 31 & $2.4\times10^{-2}$ & $2.0\times10^{-2}$ & $6.4\times10^{-4}$ & $6.9\times10^{-4}$ \\
Aircraft & ViT-B/16 & 0.56 & 0.53 & +2.61 & 32 & $2.0\times10^{-2}$ & $1.8\times10^{-2}$ & $5.6\times10^{-4}$ & $6.4\times10^{-4}$ \\
Aircraft & ViT-L/14 & 0.55 & 0.14 & +2.14 & 38 & $2.0\times10^{-2}$ & $1.8\times10^{-2}$ & $5.9\times10^{-4}$ & $6.2\times10^{-4}$ \\
Aircraft & SigLIP-B/16 & 0.57 & 0.09 & +0.53 & 44 & $1.3\times10^{-2}$ & $1.3\times10^{-2}$ & $4.2\times10^{-4}$ & $4.4\times10^{-4}$ \\
Caltech101 & RN50 & 0.53 & 0.00 & +2.40 & 30 & $3.0\times10^{-2}$ & $2.5\times10^{-2}$ & $7.3\times10^{-4}$ & $8.8\times10^{-4}$ \\
Caltech101 & ViT-B/32 & 0.59 & 0.00 & +2.35 & 29 & $2.7\times10^{-2}$ & $2.2\times10^{-2}$ & $6.6\times10^{-4}$ & $7.6\times10^{-4}$ \\
Caltech101 & ViT-B/16 & 0.59 & 0.00 & +1.84 & 29 & $2.4\times10^{-2}$ & $2.0\times10^{-2}$ & $6.6\times10^{-4}$ & $7.1\times10^{-4}$ \\
Caltech101 & ViT-L/14 & 0.58 & 0.00 & +1.30 & 32 & $2.5\times10^{-2}$ & $1.9\times10^{-2}$ & $7.1\times10^{-4}$ & $6.8\times10^{-4}$ \\
Caltech101 & SigLIP-B/16 & 0.57 & 0.00 & +1.75 & 33 & $1.8\times10^{-2}$ & $1.7\times10^{-2}$ & $5.6\times10^{-4}$ & $5.8\times10^{-4}$ \\
Cars & RN50 & 0.59 & 0.08 & +4.47 & 28 & $2.7\times10^{-2}$ & $2.6\times10^{-2}$ & $9.1\times10^{-4}$ & $9.1\times10^{-4}$ \\
Cars & ViT-B/32 & 0.61 & 0.06 & +3.74 & 28 & $2.7\times10^{-2}$ & $2.5\times10^{-2}$ & $9.0\times10^{-4}$ & $8.8\times10^{-4}$ \\
Cars & ViT-B/16 & 0.62 & 0.04 & +3.58 & 27 & $2.4\times10^{-2}$ & $2.2\times10^{-2}$ & $8.3\times10^{-4}$ & $7.8\times10^{-4}$ \\
Cars & ViT-L/14 & 0.59 & 0.02 & +2.42 & 26 & $2.5\times10^{-2}$ & $2.1\times10^{-2}$ & $8.0\times10^{-4}$ & $7.3\times10^{-4}$ \\
Cars & SigLIP-B/16 & 0.58 & 0.00 & +0.31 & 30 & $2.0\times10^{-2}$ & $1.8\times10^{-2}$ & $6.9\times10^{-4}$ & $6.1\times10^{-4}$ \\
DTD & RN50 & 0.47 & 0.38 & +5.06 & 23 & $3.3\times10^{-2}$ & $3.1\times10^{-2}$ & $1.1\times10^{-3}$ & $1.1\times10^{-3}$ \\
DTD & ViT-B/32 & 0.53 & 0.23 & +3.39 & 28 & $2.6\times10^{-2}$ & $2.5\times10^{-2}$ & $9.0\times10^{-4}$ & $8.9\times10^{-4}$ \\
DTD & ViT-B/16 & 0.51 & 0.18 & +3.71 & 26 & $2.7\times10^{-2}$ & $2.6\times10^{-2}$ & $8.7\times10^{-4}$ & $9.4\times10^{-4}$ \\
DTD & ViT-L/14 & 0.51 & 0.06 & +2.83 & 34 & $3.2\times10^{-2}$ & $2.7\times10^{-2}$ & $1.0\times10^{-3}$ & $9.9\times10^{-4}$ \\
DTD & SigLIP-B/16 & 0.50 & 0.03 & +3.52 & 30 & $2.4\times10^{-2}$ & $2.3\times10^{-2}$ & $6.9\times10^{-4}$ & $8.2\times10^{-4}$ \\
EuroSAT & RN50 & 0.38 & 2.20 & +4.85 & 11 & $1.6\times10^{-2}$ & $1.5\times10^{-2}$ & $5.4\times10^{-4}$ & $5.0\times10^{-4}$ \\
EuroSAT & ViT-B/32 & 0.40 & 1.95 & +4.50 & 11 & $2.1\times10^{-2}$ & $1.7\times10^{-2}$ & $7.6\times10^{-4}$ & $5.9\times10^{-4}$ \\
EuroSAT & ViT-B/16 & 0.43 & 1.31 & +3.76 & 12 & $1.6\times10^{-2}$ & $1.8\times10^{-2}$ & $7.2\times10^{-4}$ & $5.9\times10^{-4}$ \\
EuroSAT & ViT-L/14 & 0.40 & 0.65 & +3.13 & 22 & $2.3\times10^{-2}$ & $1.8\times10^{-2}$ & $6.9\times10^{-4}$ & $6.2\times10^{-4}$ \\
EuroSAT & SigLIP-B/16 & 0.35 & 1.15 & +4.83 & 13 & $1.3\times10^{-2}$ & $1.2\times10^{-2}$ & $4.3\times10^{-4}$ & $4.1\times10^{-4}$ \\
Flowers102 & RN50 & 0.43 & 0.00 & +0.64 & 24 & $1.2\times10^{-2}$ & $1.4\times10^{-2}$ & -- & $4.7\times10^{-4}$ \\
Flowers102 & ViT-B/32 & 0.48 & 0.00 & -0.97 & 25 & $1.1\times10^{-2}$ & $1.2\times10^{-2}$ & -- & $4.2\times10^{-4}$ \\
Flowers102 & ViT-B/16 & 0.50 & 0.00 & -1.19 & 24 & $1.1\times10^{-2}$ & $1.2\times10^{-2}$ & -- & $3.9\times10^{-4}$ \\
Flowers102 & ViT-L/14 & 0.45 & 0.00 & -2.35 & 23 & $1.1\times10^{-2}$ & $1.1\times10^{-2}$ & -- & $3.8\times10^{-4}$ \\
Flowers102 & SigLIP-B/16 & 0.42 & 0.00 & -0.60 & 22 & $8.3\times10^{-3}$ & $1.0\times10^{-2}$ & -- & $3.4\times10^{-4}$ \\
Food101 & RN50 & 0.65 & 0.06 & -0.28 & 38 & $2.5\times10^{-2}$ & $2.3\times10^{-2}$ & $1.2\times10^{-3}$ & $8.2\times10^{-4}$ \\
Food101 & ViT-B/32 & 0.67 & 0.02 & +0.10 & 38 & $2.7\times10^{-2}$ & $2.3\times10^{-2}$ & $1.1\times10^{-3}$ & $8.2\times10^{-4}$ \\
Food101 & ViT-B/16 & 0.67 & 0.01 & +0.14 & 37 & $2.6\times10^{-2}$ & $2.2\times10^{-2}$ & $1.2\times10^{-3}$ & $7.7\times10^{-4}$ \\
Food101 & ViT-L/14 & 0.64 & 0.00 & -0.16 & 45 & $2.5\times10^{-2}$ & $1.8\times10^{-2}$ & $1.2\times10^{-3}$ & $6.3\times10^{-4}$ \\
Food101 & SigLIP-B/16 & 0.61 & 0.00 & +0.11 & 41 & $1.9\times10^{-2}$ & $1.5\times10^{-2}$ & $8.2\times10^{-4}$ & $5.3\times10^{-4}$ \\
ImageNet & RN50 & 0.60 & 0.00 & +2.02 & 33 & $3.3\times10^{-2}$ & $3.0\times10^{-2}$ & $1.1\times10^{-3}$ & $1.1\times10^{-3}$ \\
ImageNet & ViT-B/32 & 0.68 & 0.00 & +1.30 & 34 & $3.3\times10^{-2}$ & $2.8\times10^{-2}$ & $1.1\times10^{-3}$ & $1.0\times10^{-3}$ \\
ImageNet & ViT-B/16 & 0.68 & 0.00 & +1.24 & 33 & $3.0\times10^{-2}$ & $2.6\times10^{-2}$ & $1.0\times10^{-3}$ & $9.4\times10^{-4}$ \\
ImageNet & ViT-L/14 & 0.68 & 0.00 & +0.78 & 38 & $2.8\times10^{-2}$ & $2.4\times10^{-2}$ & $1.0\times10^{-3}$ & $8.7\times10^{-4}$ \\
ImageNet & SigLIP-B/16 & 0.64 & 0.00 & +0.61 & 39 & $2.3\times10^{-2}$ & $2.0\times10^{-2}$ & $7.5\times10^{-4}$ & $7.1\times10^{-4}$ \\
Pets & RN50 & 0.57 & 0.10 & +0.72 & 34 & $2.2\times10^{-2}$ & $1.9\times10^{-2}$ & $9.1\times10^{-4}$ & $6.4\times10^{-4}$ \\
Pets & ViT-B/32 & 0.59 & 0.06 & +0.33 & 35 & $2.4\times10^{-2}$ & $1.9\times10^{-2}$ & $8.8\times10^{-4}$ & $6.7\times10^{-4}$ \\
Pets & ViT-B/16 & 0.56 & 0.04 & +0.72 & 31 & $2.1\times10^{-2}$ & $1.8\times10^{-2}$ & $7.3\times10^{-4}$ & $6.3\times10^{-4}$ \\
Pets & ViT-L/14 & 0.50 & 0.01 & +0.14 & 35 & $2.4\times10^{-2}$ & $1.9\times10^{-2}$ & $8.0\times10^{-4}$ & $6.4\times10^{-4}$ \\
Pets & SigLIP-B/16 & 0.52 & 0.01 & -0.19 & 36 & $1.7\times10^{-2}$ & $1.3\times10^{-2}$ & $5.4\times10^{-4}$ & $4.4\times10^{-4}$ \\
SUN397 & RN50 & 0.58 & 0.01 & +3.04 & 31 & $3.5\times10^{-2}$ & $3.2\times10^{-2}$ & $1.5\times10^{-3}$ & $1.2\times10^{-3}$ \\
SUN397 & ViT-B/32 & 0.64 & 0.00 & +3.16 & 32 & $3.3\times10^{-2}$ & $2.8\times10^{-2}$ & $1.3\times10^{-3}$ & $1.0\times10^{-3}$ \\
SUN397 & ViT-B/16 & 0.63 & 0.00 & +3.35 & 32 & $3.2\times10^{-2}$ & $2.7\times10^{-2}$ & $1.4\times10^{-3}$ & $9.8\times10^{-4}$ \\
SUN397 & ViT-L/14 & 0.63 & 0.00 & +2.23 & 39 & $3.0\times10^{-2}$ & $2.5\times10^{-2}$ & $1.3\times10^{-3}$ & $8.9\times10^{-4}$ \\
SUN397 & SigLIP-B/16 & 0.59 & 0.00 & +1.20 & 38 & $2.3\times10^{-2}$ & $2.0\times10^{-2}$ & $8.4\times10^{-4}$ & $6.8\times10^{-4}$ \\
 
\bottomrule
\end{tabular}
\end{table*}

\section{The Modern-Encoder Arm}
The robustness arm of Sec.~6.8 of the main text repeats the analysis on encoders of the current generation; Tables~\ref{tab:siglip2} and~\ref{tab:probes} give the per-shot detail behind the two sentences the main text spends on it. Both arms run the full ten-dataset grid (five seeds; the standard thin-pool drop rule leaves $240$ cells each), and every interval is the paper's cluster bootstrap over datasets with the Student-$t$ correction.

\emph{SigLIP2-B/16}~\cite{tschannen2025siglip2} carries a text tower, so the blending family is defined on it and the standard runner applies unchanged (\texttt{code/run\_formal.py}). Table~\ref{tab:siglip2} repeats the paper's central contrasts: the James--Stein blend trails the test-set-oracle ratio at every shot count while the LOO blend tracks it, and CLAP clears the oracle with an interval excluding zero. LP++ attenuates on this encoder: its $K{\ge}4$ margin ($+1.50$) sits below the published five-backbone $K{\ge}4$ figure ($+2.78$), and pooled over all $K$ its point estimate is negative ($-0.16$ \CI{-1.45}{+1.14}, dragged by $K{=}1$: $-4.12$ \CI{-6.06}{-2.17}); a single backbone's $140$ cells over ten clusters resolve none of these.

\emph{DINOv3-B/16}~\cite{simeoni2025dinov3} and \emph{DINOv2-B/14}~\cite{oquab2024dinov2} have no text tower: no text prototype exists, so the blending ratio of Eq.~(1) is undefined for them and they can enter only as probe-side reference points (\texttt{code/run\_vision\_only.py}). The probe is a multinomial logistic regression on frozen CLS features whose regularisation is reported both at its own test-set oracle (the apples-to-apples comparison against the family's test-set-oracle ratio) and selected validation-free by leave-one-shot-out on the support set---reporting a probe at one arbitrary setting is the practice the main text criticises. Labels, splits and support draws are shared with the CLIP side cell for cell, so every difference in Table~\ref{tab:probes} is paired. At $86$M parameters DINOv3-B/16 matches CLIP ViT-B/16's budget, isolating the pre-training objective from model size. The DINOv3 weights are gated on the Hub, so the pipeline may take them from a public mirror; \texttt{code/check\_dinov3.py} first asserts the architecture, parameter count, preprocessing geometry (DINOv3 declares a direct $224{\times}224$ resize with no centre crop, unlike DINOv2's resize-then-crop) and feature sanity before any experiment may depend on the checkpoint.

\begin{table*}[t]
\centering
\caption{SigLIP2-B/16: the paper's central contrasts, per shot count (pp versus the test-set-oracle blending ratio, paired within cells; cluster bootstrap over datasets, Student-$t$). Generated by \texttt{code/modern\_macros.py} from the released records.}
\label{tab:siglip2}
\footnotesize
\setlength{\tabcolsep}{5pt}
\begin{tabular}{l cccc}
\toprule
 & James--Stein blend & LOO blend & CLAP & LP++ \\
\midrule
$K=1$ & -- & -- & $+0.30$ \CI{-0.49}{+1.08} & $-4.12$ \CI{-6.06}{-2.17} \\
$K=2$ & $-11.57$ \CI{-16.30}{-6.84} & $-1.32$ \CI{-2.43}{-0.21} & $+1.22$ \CI{-0.08}{+2.52} & $-0.84$ \CI{-2.35}{+0.68} \\
$K=4$ & $-5.98$ \CI{-8.74}{-3.21} & $-1.02$ \CI{-1.99}{-0.04} & $+1.48$ \CI{+0.17}{+2.79} & $+0.67$ \CI{-0.67}{+2.01} \\
$K=8$ & $-3.87$ \CI{-5.73}{-2.02} & $-0.56$ \CI{-0.90}{-0.22} & $+1.77$ \CI{+0.42}{+3.12} & $+1.73$ \CI{-0.16}{+3.62} \\
$K=16$ & $-2.51$ \CI{-3.88}{-1.13} & $-0.48$ \CI{-0.79}{-0.18} & $+1.80$ \CI{+0.44}{+3.17} & $+2.20$ \CI{+0.10}{+4.29} \\
\midrule
$K\ge4$ pooled & $-4.19$ \CI{-6.08}{-2.29} & $-0.70$ \CI{-1.21}{-0.18} & $+1.68$ \CI{+0.37}{+2.98} & $+1.50$ \CI{-0.22}{+3.22} \\
all $K$ & $-6.13$ \CI{-8.63}{-3.63} & $-0.86$ \CI{-1.50}{-0.22} & $+1.29$ \CI{+0.17}{+2.42} & $-0.16$ \CI{-1.45}{+1.14} \\
 
\bottomrule
\end{tabular}
\end{table*}

\begin{table*}[t]
\centering
\caption{Text-free reference points: a plain logistic probe on self-supervised features versus the blending family's test-set-oracle ratio (pp, paired within cells). ``Oracle reg.''\ scores the probe at its best test-set regularisation---supremum against supremum; ``LOO reg.''\ selects it validation-free on the support set. The same probe on CLIP's own features is the control: the family's ceiling is cleared by the text-free encoder, not by the probe construction.}
\label{tab:probes}
\footnotesize
\setlength{\tabcolsep}{5pt}
\begin{tabular}{l cccc}
\toprule
 & DINOv3-B/16 (oracle reg.) & DINOv3-B/16 (LOO reg.) & DINOv2-B/14 (oracle reg.) & CLIP ViT-B/16 (oracle reg.) \\
\midrule
$K=1$ & $-9.03$ \CI{-18.12}{+0.07} & -- & $-14.30$ \CI{-25.09}{-3.51} & $-23.78$ \CI{-33.64}{-13.92} \\
$K=2$ & $-0.66$ \CI{-7.60}{+6.28} & $-2.39$ \CI{-9.11}{+4.33} & $-7.00$ \CI{-15.20}{+1.21} & $-13.91$ \CI{-21.29}{-6.54} \\
$K=4$ & $+4.55$ \CI{-1.56}{+10.66} & $+3.56$ \CI{-2.56}{+9.68} & $-2.01$ \CI{-7.60}{+3.58} & $-6.51$ \CI{-11.58}{-1.43} \\
$K=8$ & $+7.62$ \CI{+0.38}{+14.85} & $+6.54$ \CI{-0.82}{+13.90} & $+1.57$ \CI{-2.88}{+6.01} & $-1.74$ \CI{-5.63}{+2.16} \\
$K=16$ & $+9.76$ \CI{+2.06}{+17.46} & $+8.87$ \CI{+1.10}{+16.64} & $+5.33$ \CI{+0.23}{+10.42} & $+1.99$ \CI{-1.31}{+5.29} \\
\midrule
$K\ge4$ pooled & $+7.21$ \CI{+0.41}{+14.01} & $+6.22$ \CI{-0.62}{+13.07} & $+1.50$ \CI{-2.98}{+5.97} & $-2.24$ \CI{-6.25}{+1.76} \\
 
\bottomrule
\end{tabular}
\end{table*}

\end{document}